\def\year{2020}\relax
\documentclass[letterpaper]{article} 
\usepackage{aaai20}  
\usepackage{times}  
\usepackage{helvet} 
\usepackage{courier}  
\usepackage[hyphens]{url}  
\usepackage{graphicx} 
\usepackage{graphicx}  
\usepackage{amsmath,amssymb}

\usepackage{amsthm}
\usepackage{xspace}
\usepackage{stmaryrd}
\usepackage{wasysym}
\usepackage{microtype}
\usepackage{array}
\usepackage[colorinlistoftodos,obeyFinal]{todonotes}
\usepackage{booktabs}
\usepackage{csquotes}
\usepackage{enumitem}

\usepackage{tikz}
\usetikzlibrary{backgrounds}
\usetikzlibrary{calc}
\usetikzlibrary{positioning,fit}
\usepgfmodule{plot}

\usepackage{tikz}
\definecolor{Black}  {RGB}{0,0,0}

\pgfdeclarelayer{edgelayer}
\pgfdeclarelayer{nodelayer}
\pgfsetlayers{edgelayer,nodelayer,main}

\tikzstyle{none}=[inner sep=0pt]

\tikzstyle{rn}=[circle,fill=Red,draw=Black,line width=0.8 pt]
\tikzstyle{gn}=[circle,fill=White,draw=Black,line width=0.8 pt]
\tikzstyle{yn}=[circle,fill=Yellow,draw=Black,line width=0.8 pt]
\tikzstyle{simple}=[circle,fill=White,draw=Black]
\tikzstyle{newstyle1}=[circle,fill=Black,draw=Black,line width=0.3 pt,inner sep=0pt]
\tikzstyle{simple2}=[-,dashed,draw=Black]
\tikzstyle{simpledotted}=[-,dotted,draw=Black]
\tikzstyle{simple}=[-,draw=Black,line width=2.000]
\tikzstyle{arrow}=[-,draw=Black,postaction={decorate},decoration={markings,mark=at position .5 with {\arrow{>}}},line width=2.000]
\tikzstyle{tick}=[-,draw=Black,postaction={decorate},decoration={markings,mark=at position .5 with {\draw (0,-0.1) -- (0,0.1);}},line width=2.000]
\tikzstyle{newstyle2}=[-latex,draw=Black]
\tikzstyle{newstyle3}=[->,dotted,draw=Black]
\tikzstyle{newstyle6}=[->,dotted,draw=Black]

 \usepackage{boxedminipage}
\usepackage[ruled,vlined]{algorithm2e}
\usepackage{mathtools}
\usepackage{color} 
\usepackage{thm-restate} 
\usepackage{colortbl}
\usepackage{pifont}
\usepackage{array}
\usepackage{multirow}
\usepackage{scalerel,stackengine}

\newcommand{\mypar}[1]{\medskip\noindent\textbf{#1.}}

\newcommand{\Nbb}{\ensuremath{\mathbb{N}}\xspace}

\newcommand{\individuals}[1]{\mathsf{ind}(#1)}

\def\rConstrHelper(#1,#2,#3,#4){{#1}\ \substack{#2 \\ #3}\ {#4}}

\newcommand{\sample}{\ensuremath{S}\xspace}
\newcommand{\prob}{\ensuremath{\Dmc}\xspace}
\newcommand{\batch}{\ensuremath{\Bmc}\xspace}

\newcommand{\indiv}{\ensuremath{\mathsf{ind}}\xspace}
\newcommand{\homo}{\ensuremath{h}\xspace}
\newcommand{\homsubset}{\ensuremath{\Xmc}\xspace}
\newcommand{\poly}{\ensuremath{\mathit{p}}\xspace}
\def\apeqA{\SavedStyle\sim}
\def\apeq{\setstackgap{L}{\dimexpr.5pt+1.5\LMpt}\ensurestackMath{
  \ThisStyle{\mathrel{\Centerstack{{\apeqA} {\apeqA} {\apeqA}}}}}}
\newcommand{\con}[2]{\ensuremath{{#1}\apeq{#2}}\xspace}
\newcommand{\ssigma}{\ensuremath{\text{{\boldmath $\sigma$}}}\xspace}

\newcommand{\sss}{{\mathrel{\kern.25em{\sqsubseteq}\kern-.5em \mbox{{\scriptsize *}}\kern.25em}}}
\newcommand{\smallsss}{{\mathrel{\kern.25em{\sqsubseteq}\kern-.4em \mbox{{\scriptsize *}}\kern.25em}}}

\newcommand{\T}{\ensuremath{\mathcal T}\xspace}

\newcommand{\I}{\ensuremath{\mathcal{I}}\xspace}

\newcommand{\node}{\ensuremath{\mathop{\mathsf{node}}}}

\newcommand{\ELH}{\ensuremath{{\cal E\!LH}}\xspace}
\newcommand{\EL}{\ensuremath{{\cal E\!L}}\xspace}

\newcommand{\NC}{\ensuremath{{\sf N_C}}\xspace}

\newcommand{\NI}{\ensuremath{{\sf N_I}}\xspace}

\newcommand{\NR}{\ensuremath{{\sf N_R}}\xspace}

\newlength{\indxlength}
\newcommand{\PTime}{\textsc{PTime}\xspace}
\newcommand{\NP}{\textsc{NP}\xspace}

\newcommand{\PTimeL}{\textsc{PTimeL}\xspace}
\newcommand{\PQuery}{\textsc{PQueryL}\xspace}

\newcommand{\Amc}{\ensuremath{\mathcal{A}}\xspace}
\newcommand{\Bmc}{\ensuremath{\mathcal{B}}\xspace}

\newcommand{\Dmc}{\ensuremath{\mathcal{D}}\xspace}
\newcommand{\Emc}{\ensuremath{\mathcal{E}}\xspace}

\newcommand{\Hmc}{\ensuremath{\mathcal{H}}\xspace}
\newcommand{\Imc}{\ensuremath{\mathcal{I}}\xspace}
\newcommand{\Jmc}{\ensuremath{\mathcal{J}}\xspace}
\newcommand{\Kmc}{\ensuremath{\mathcal{K}}\xspace}
\newcommand{\Lmc}{\ensuremath{\mathcal{L}}\xspace}

\newcommand{\Qmc}{\ensuremath{\mathcal{Q}}\xspace}
\newcommand{\Rmc}{\ensuremath{\mathcal{R}}\xspace}
\newcommand{\Smc}{\ensuremath{\mathcal{S}}\xspace}
\newcommand{\Tmc}{\ensuremath{\mathcal{T}}\xspace}

\newcommand{\Xmc}{\ensuremath{\mathcal{X}}\xspace}

\newcommand{\Zmc}{\ensuremath{\mathcal{Z}}\xspace}

\newcommand{\Emf}{\ensuremath{\mathfrak{E}}\xspace}
\newcommand{\Fmf}{\ensuremath{\mathfrak{F}}\xspace}

\newcommand{\Lmf}{\ensuremath{\mathfrak{L}}\xspace}

\newcommand{\Vmf}{\ensuremath{\mathfrak{V}}\xspace}

 \newcommand{\cellgray}{\cellcolor{gray!30}} 

\newcommand{\examples}{\ensuremath{\Emc}\xspace} 
\newcommand{\setexamples}{\ensuremath{\Xmc}\xspace} 
\newcommand{\vc}[1]{\ensuremath{{\sf VC}({#1})}\xspace}
\newcommand{\target}{\ensuremath{t}\xspace} 
\newcommand{\hypothesisSpace}{\ensuremath{\Lmc}\xspace} 
\newcommand{\fixedAbox}{\ensuremath{\Amc_\ast}\xspace} 
\newcommand{\ontologyLanguage}{\ensuremath{\Lmf}\xspace} 
\newcommand{\hypothesisOntology}{\ensuremath{\Hmc}\xspace} 
\newcommand{\targetOntology}{\ensuremath{\Tmc}\xspace} 
\newcommand{\queryLanguage}{\ensuremath{\Qmc}\xspace} 
\newcommand{\query}{\ensuremath{q}\xspace} 
\newcommand{\e}{\ensuremath{e}\xspace} 
\newcommand{\A}{\ensuremath{\mathcal{A}}\xspace}
\newcommand{\AQ}{\ensuremath{\textnormal{AQ}}\xspace}
\newcommand{\IQ}{\ensuremath{\textnormal{IQ}}\xspace}
\newcommand{\CQ}{\ensuremath{\textnormal{CQ}}\xspace}
\newcommand{\CQr}{\ensuremath{\textnormal{CQ}\ensuremath{_{{ r}}}}\xspace}
\newcommand{\CQrs}{CQ\ensuremath{_{{ r}}}s\xspace}
\newcommand{\AQs}{AQs\xspace}
\newcommand{\IQs}{IQs\xspace}
\newcommand{\CQs}{CQs\xspace}
\newcommand{\K}{\ensuremath{\mathcal{K}}\xspace}

\newcommand{\EX}{\ensuremath{{\sf EX}^{\prob}_{\Fmf,\target}}\xspace}

\newcommand{\set}{\ensuremath{\Smc}\xspace}

\newcommand{\paths}{\ensuremath{{\sf paths}}\xspace}
\newcommand{\tail}{\ensuremath{{\sf tail}}\xspace}
\newcommand{\refLemma}[1]{Lemma~#1}
\newcommand{\tessential}{\T-essential\xspace}

\newcommand{\treeV}[1]{\ensuremath{\Vmf_{#1}}\xspace}
\newcommand{\treeE}[1]{\ensuremath{\Emf_{#1}}\xspace}
\newcommand{\treel}[1]{\ensuremath{l_{#1}}\xspace}
\newcommand{\treedef}[1]{\ensuremath{T_{#1}=(\treeV{#1},\treeE{#1},\treel{#1})}\xspace}
\newcommand{\tree}[1]{\ensuremath{T_{#1}}\xspace}

\newcommand{\cmark}{\ding{51}\xspace}
\newcommand{\xmark}{\ding{55}\xspace}

\newtheorem{theorem}{Theorem}
\newtheorem{definition}[theorem]{Definition}

\newtheorem{example}[theorem]{Example}
\newtheorem{remark}{Remark}
\newtheorem{lemma}[theorem]{Lemma}
\newtheorem{proposition}[theorem]{Proposition}
\newtheorem{claim}{Claim} 
\newtheorem{claimt}{Claim} 
\newcommand{\lab}[2]{\ensuremath{\ell_{#2}(#1)}\xspace}
\newif\iftechrep
\techreptrue

\begin{document}
\title{Learning Query Inseparable \ELH Ontologies}
\author{Ana Ozaki, Cosimo Persia,  Andrea Mazzullo\\
KRDB Research Centre, Free University of Bozen-Bolzano}

\maketitle

\begin{abstract}
We investigate the complexity of learning query inseparable \ELH ontologies
in a variant of Angluin's exact learning model. 
Given a fixed data instance \fixedAbox 
and a query language \queryLanguage, we are interested 
in computing an ontology \hypothesisOntology that entails the same queries 
as a target ontology \targetOntology on \fixedAbox, that is, 
\hypothesisOntology and \targetOntology are 
inseparable w.r.t. \fixedAbox  and \queryLanguage.
The learner is allowed to pose two kinds of questions. 
The first is `Does $(\targetOntology,\A)\models q$?', with 
\A an arbitrary data instance and $q$ and query in \queryLanguage. 
An oracle replies this question with `yes' or `no'.
In the second, the learner asks `Are \hypothesisOntology and \targetOntology  
inseparable w.r.t. \fixedAbox  and \queryLanguage?'. 
If so, the learning process finishes, otherwise, the learner receives $(\fixedAbox,q)$ 
with  $q\in \queryLanguage$,  
 $(\targetOntology,\fixedAbox)\models q$ and $(\hypothesisOntology,\fixedAbox)\not\models q$ (or vice-versa).
Then, 
we analyse   conditions in which 
query inseparability is preserved if \fixedAbox  changes. 
Finally, 
we consider the PAC learning model and a setting where the algorithms
learn from 
a batch of classified data, 
 limiting  interactions with the oracles.
\end{abstract}

\section{Introduction}

 Ontologies are a formal and popular way of representing knowledge. 
Taxonomies, categorisation of websites, products and their features, as well as 
more complex and specialized domain knowledge, 
can be represented with ontologies. 
Domain experts use ontologies while sharing and annotating information in their fields because in this way 
knowledge can be unambiguously understood and easily distributed. Medicine, for example, has produced large, 
standardised, and scalable ontologies (e.g. Galen and SNOMED CT). 
Broad and general so called knowledge graphs 
are emerging such as DBPedia~\cite{dbpedia09}, Wikidata~\cite{VK:wikidata14}, 
 YAGO \cite{SKW08:yago}. 
An ontology 
enables machines to process relations and definitions and reason about that knowledge.
Sharing information, formalising a domain, and making 
assumptions explicit 
are some of
the main reasons for using an ontology.

Designing ontologies is a hard and error-prone task. 
The research community has approached the problem   by developing editors 
that help ontology engineers to build  ontologies manually~\cite{musen2013protege} and 
defined   design principles \cite{DBLP:series/lncs/5445}. 
Even with tools, building ontologies is a laborious task that also needs  expertise. 
An expert in designing ontologies is  
called an ontology engineer. Such an expert is normally familiar with the tools, 
   languages, and techniques necessary to design an ontology 
through the communication with domain experts. 
The ontology engineer must communicate and understand the knowledge 
given by domain experts and then design an ontology that captures what is relevant in the domain. 

One of the main challenges in the process of building an ontology 
is that it often 
 relies on the communication between the ontology engineer (or multiple ontology engineers)
  and domain experts that, 
in order to share knowledge, use the ambiguous natural language. Indeed, 
it can happen that some errors are made while designing it due to the 
difficulty of sharing knowledge. The ontology engineer can misunderstand 
the domain experts or they can inadvertently omit   precious details. 
Moreover, knowledge  can be implicit and while designing an ontology it is 
not easy to understand 
what are the real relationships between concepts 
and imagine all the possible   consequences of designing concepts and their relationships 
in a particular way. 
There are several aspects which can influence the difficulty of 
creating an ontology.

Following the approach by~\cite{KLOW18,DBLP:conf/aaai/KonevOW16}, 
 we focus on  the problem of finding how concepts should be logically 
related, assuming that the relevant vocabulary is known by the domain experts, which 
can share this information with the ontology engineer. 
In this approach,  the problem of building an ontology  is treated as a 
learning problem in which the ontology engineer plays the role of a learner and
 communicates with the domain experts, who play the role of a teacher (also called an oracle).  
We assume that (1) the domain 
experts  know 
the relevant knowledge about the domain and act 
in a consistent way as a single teacher;
(2)  the 
vocabulary that should be used in the ontology is known by the teacher and the learner; 
(3) the learner can pose queries to the teacher in 
order to acquire missing knowledge or check if it has learned enough in order to stop learning.
The described model can be seen as an instance of Angluin's exact learning model~\cite{angluinqueries} 
with membership and equivalence queries. 
The queries asked by the learner in order to acquire 
knowledge can be considered as membership queries and 
the queries that ask if the hypothesis of the learner correctly represents the relevant knowledge 
of the domain experts can be treated as equivalence queries.

\begin{table}[]
\centering

\begin{tabular}{|l|c|c|c|}
\hline
\small{Framework} & \small{$\EL(\Hmc)_{\sf lhs}$} & \small{$\EL(\Hmc)_{\sf rhs}$} & \small{$\EL(\Hmc)$}    \\ 
\hline 
 \parbox[t]{2mm}{\multirow{3}{*}{\rotatebox[origin=c]{90}{Equiv. }}}  
 $\text{ }$\small{\ AQ} & \cmark  & --  & --   \\ 
$\quad$\small{\ IQ}   & \cmark  & \cmark  & \xmark   \\ 
$\quad$\small{\ \CQr }   & \cellgray \cmark  & \cellgray \cmark  & ?   \\ 
$\quad$\small{\ CQ} & \cmark  & \xmark   & \xmark   \\ 
\hline 
 \parbox[t]{2mm}{\multirow{3}{*}{\rotatebox[origin=c]{90}{Insep. }}}  
 $\text{ }$\small{\ AQ} & \cellgray \cmark  & \cellgray \cmark  & \cellgray \cmark   \\ 
$\quad$\small{\ IQ}   & \cellgray\cmark  & \cellgray\cmark  &\cellgray \cmark   \\ 
$\quad$\small{\ \CQr}   & \cellgray\cmark  &\cellgray \cmark  &\cellgray \cmark   \\ 
$\quad$\small{\ CQ} &\cellgray \cmark  &\cellgray \xmark   &\cellgray \xmark    \\ 
\hline
 \parbox[t]{2mm}{\multirow{3}{*}{\rotatebox[origin=c]{90}{PAC   }}}  
 $\text{ }$\small{\ AQ} & \cmark  & \cmark  & \cmark   \\ 
$\quad$\small{\ IQ}   & \cmark  &\cmark  &\cmark  \\ 
$\quad$\small{\ \CQr}   & \cmark  &\cmark  &\cmark   \\ 
$\quad$\small{\ CQ} & \cmark  &\cmark   &\cmark   \\ 
 \hline 
\end{tabular} 
\caption{\small{Polynomial query learnability of learning frameworks.}}\label{tab:pos-neg}
\end{table}

In the exact learning literature, one of the main goals is to determine the complexity of  
learning an abstract target. 
The complexity depends on 
the number and size of queries 
posed by the learner. 
In the work by~\cite{KLOW18,DBLP:conf/aaai/KonevOW16} 
it has been shown 
that exactly learnability of ontologies formulated in the very popular \ELH~\cite{dlhandbook}  ontology language 
is not possible with polynomially many polynomial size queries.  
Here, we investigate a more flexible setting, where
the ontology does not need to be logically equivalent 
but only inseparable w.r.t. a query language~\cite{DBLP:journals/jsc/LutzW10} 
 and a fixed 
data instance\footnote{ 
Here 
the term `query' refers to queries in the context of databases 
and query answering. Our data instances are ABoxes.}.
Query inseparability is the basic requirement 
for ontology mediated query answering (OMQA)~\cite{DBLP:conf/ijcai/Bienvenu16}. 
In OMQA, relevant tasks such as ontology versioning, modularisation, update, and forgetting, 
depend on comparisons between ontologies based on answers given to queries~\cite{DBLP:journals/ai/BotoevaLRWZ19}.
We study polynomial query and time learnability of the \ELH ontology language 
in the OMQA setting. 
The query languages considered 
are atomic queries (AQ), instance queries (IQ), conjunctive queries (CQ) and a fragment of CQs 
called rooted CQs (denoted \CQr).
In particular, we show that in our setting the picture is brighter and 
\ELH can be polynomially learned from IQs, however, it is still 
not possible to learn this language from CQs. 

Table \ref{tab:pos-neg} shows the different results obtained by previous 
works~\cite{KLOW18,DBLP:conf/aaai/KonevOW16}  for logical equivalence and our results (shaded in gray) for query inseparability, 
taking into account the ontology  (upper side) and the query languages (left side).  \cmark  
means a positive result, i.e. polynomial query learnability; 
 $ - $ means that the query language is not expressive enough for exchanging 
 information and \xmark means that polynomial query learnability cannot be achieved.
Polynomial time 
learnability implies polynomial query learnability but the converse does not hold. 
All positive results for \AQs, \IQs, and \CQrs also hold for polynomial time learnability.
Since the learned ontology is not equivalent to the target, it may be the case 
that after the data is updated the learned ontology and the target are no longer 
query inseparable. 
We thus investigate conditions under which query inseparability is preserved when the 
data changes, which means that no further learning steps are needed after the change. 
In many applicaton scenarios, interactions with teachers may 
not be a viable option.
We also  consider learnability when 
the learner has only access to 
a batch of classified examples. Finally, we adapt  the Probably Approximately Correct (PAC) model~\cite{Valiant} to our OMQA setting, which we separate 
from the exact and query inseparable problem settings (Theorem~\ref{thm:pacnotimplyexact}).
Our polynomial time results for query inseparability are transferable to the PAC model extended 
with membership queries (Theorem~\ref{thm:paclearn}).  
Omitted proofs are available 
\iftechrep
in the appendix.
\else
at \url{https://arxiv.org/abs/1911.07229}.
\fi

\mypar{Related Work}
To cover the vast literature on ontology learning,
we point to the collection edited by~\cite{lehmann2014perspectives} and
surveys authored by~\cite{DBLP:reference/nlp/CimianoVB10} 
and~\cite{Wong:2012:OLT:2333112.2333115}. 
The closest works are the already mentioned papers on exact learning 
of lightweight
description logics (DLs)~\cite{KLOW18,DBLP:conf/kr/DuarteKO18,DBLP:conf/aaai/KonevOW16}.
Exact learning of concepts formulated in the DL CLASSIC has been 
investigated by~\cite{Cohen94learningthe} and~\cite{DBLP:journals/ml/FrazierP96}.
Some other works 
which are more closely related include 
 works on learning \EL concepts~\cite{DBLP:conf/ijcai/FunkJLPW19,lehmann2009ideal}. 
Formal Concept Analysis has   been applied 
for learning
DL
ontologies~\cite{DBLP:conf/iccs/Rudolph04,DBLP:conf/ijcai/BaaderGSS07,BorDi11,borchmann2014learning,ganter2016conceptual}.
Learnability of \EL ontologies from finite interpretations has also been investigated~\cite{klarmanontology}. 
Association rule mining has   been used to learn DL ontologies (with concept 
expressions of limited depth)~\cite{DBLP:conf/semweb/SazonauS17,volker2011statistical,DBLP:conf/otm/FleischhackerVS12,DBLP:journals/ws/VolkerFS15}.

\section{Basic Definitions}
\mypar{Ontologies and Queries}
The \ELH syntax is defined upon   mutually disjoint countably infinite sets of 
concept names \NC, denoted with $ A,B $, role names \NR, denoted with $ r,s $, 
and individual names \NI, denoted with $ a,b $. \emph{\EL-concept expressions} $ C $ are defined 
inductively according to the  rule 
$C ::= A \mid \top \mid C \sqcap C \mid  \exists r.C$,
where $A\in\NC$ and $r\in \NR$. 
For simplicity, we omit \EL- from  \EL-concept expressions. 
An \ELH ontology, also called \emph{TBox}, is a finite set of \emph{concept inclusions} (CI) $ C\sqsubseteq D $, where $ C,D $ are concept expressions, 
 and   \emph{role inclusions} (RI) $ r\sqsubseteq s $, where $r,s\in\NR$.
We call an \ELH TBox $\Tmc$ a \emph{terminology} if for all $C \sqsubseteq
D\in\Tmc$ either $C$ or $D$ is a concept name\footnote{
In the literature, the term \emph{terminology} commonly refers to sets of
concept inclusions $A\sqsubseteq C$ and 
concept definitions $A\equiv C$, with no concept name occurring 
more than once on the left. As $A\equiv C$ can be 
equivalently rewritten as $A\sqsubseteq C$ and 
$C\sqsubseteq A$, our definition is a natural extension of this one.} and \Tmc has at most one\footnote{
If a terminology contains $A\sqsubseteq C$ and $A\sqsubseteq D$ one can always 
rewrite it into  $A\sqsubseteq C\sqcap D$. 
}
inclusion of the form $A\sqsubseteq C$ for every $A\in\NC$.
From now on we assume all \ELH TBoxes we deal with are terminologies. 
An \emph{ABox} \A is a finite set of expressions of the form $ A(a) $ or $ r(a,b) $, called \emph{assertions}, with $ A\in\NC $, $ r\in\NR $ and $ a,b\in\NI$.
We denote by $\individuals{\A}$ the set of individual names occurring in an ABox \A.
The \emph{signature} $\Sigma_\Tmc$ of a TBox \Tmc is the set of concept and role names 
 occurring in it, and similarly for the signature $\Sigma_\Amc$ of an ABox $\Amc$.
A \emph{knowledge base} (KB) is a pair (\T,\A) where \T is a TBox and \A is an ABox.
We investigate classical query languages considered in the OMQA literature.
An AQ 
takes the form of an assertion. 
An 
IQ is of the form $C(a)$ or $r(a,b)$, where
$C$ is a concept expression, $r\in \NR$ and $a,b\in \NI$. 
A CQ 
is a first-order sentence $\exists \vec{x}\varphi(\vec{a},\vec{x})$,
where $\varphi$ is a conjunction of atoms of the form $r(t_{1},t_{2})$ or $A(t)$, where $t_{1},t_{2},t$
(called \emph{terms}) 
can be individual names from $\vec{a}$ or individual variables from $\vec{x}$. 
With an abuse of notation, we denote by $\AQ$, $\IQ$ and $\CQ$ the sets of atomic, instance 
and conjunctive queries $q$, respectively, and we call a \emph{query language} a set $\Qmc \in \{ \AQ, \IQ, \CQ \}$.
The \emph{size} of a concept expression $C$ 
(TBox $\Tmc$, ABox $\Amc$, query $q$), denoted by $|C|$ (and, respectively,
$|\Tmc|$, $|\Amc|$, $|q|$) is the length of the string that represents it,  
where concept, role, and individual names are considered to be of length one. 

The semantics of $\ELH$ is given as follows.
An interpretation is a pair $ \I=(\Delta^\I,\cdot^\I) $, where 
$ \Delta^\I $ is a non-empty set, called \emph{domain}, and $ \cdot^\I $ is a function that 
maps every   $ a\in\NI $ to
$ a^\I\in\Delta^\I $, every $ A\in\NC $ to 
$A^\I \subseteq \Delta^\I $ and every $ r\in\NR $ to
$r^\I \subseteq \Delta^\I\times\Delta^\I $.
The function $  \cdot^\I $ extends to
other
concept expressions as follows:
$\top^\I = \Delta^\I$,
$(C \sqcap D)^\I := C^\I \cap D^\I$ 
and
$(\exists r.C)^\I := \{d \in \Delta^\I \mid$
$
\text{there is }
e \in C^\I
\colon
(d,e) \in r^\I
\}$.
An interpretation $ \I $ \emph{satisfies}
a CI $ C\sqsubseteq D $ 
if $ C^\I \subseteq D^\I $, and an RI $ r\sqsubseteq s $ if $ r^\I \subseteq s^\I  $. 
It satisfies an assertion $ A(a) $ if $ a^\I\in A^\I $, 
and an assertion $ r(a,b) $ if $ (a^\I,b^\I)\in r^\I $.
$ \I $ satisfies an \AQ if it satisfies the corresponding 
assertion,
and it satisfies an \IQ $C(a)$, or $r(a,b)$, if $a^\I\in C^\I$, or $(a^\I,b^\I)\in r^\I$.
\I satisfies a \CQ $q$ (or there is a homomorphism from $q$ to \I) if there is 
a function $\pi$,
mapping terms of $q$ to elements of $\Delta^\I$, 
such that:
$\pi(t)=t^\I$, if $t\in\NI$;
$\pi(t) \in A^\I$, for every $A(t)$ of $q$;
and $(\pi(t_1),\pi(t_2))\in r^\I$, for every $r(t_1,t_2)$ of $q$.
We write $\I \models \alpha$ to state that $\I$ satisfies a CI, RI, assertion, or query $\alpha$.
$ \I $ satisfies a TBox \T,
if it is a model of every CI and RI in $\T$,
and it
satisfies an ABox \A if it satisfies every assertion in $\A$.
$ \I $
satisfies
a KB $ \Kmc =  (\T,\A) $, written $\I \models \Kmc$, if it satisfies both \T and \A.
A KB \Kmc \emph{entails} a CI, an RI, an assertion, or query $\alpha$, written $\Kmc\models \alpha$,
if, for all interpretations \I, $\Imc\models\Kmc$ implies $\I\models \alpha$.
A KB \K entails a KB $\K'$, written $\K\models\K'$, 
 if $\I\models\K$ implies $\I\models\K'$;
$\K$ and $\K'$ are 
 \emph{equivalent}, in symbols $\K\equiv\K'$, if $\K\models\K'$ and $\K'\models\K$.
We may also speak of entailments and equivalences of TBoxes and ABoxes, defined as usual~\cite{dlhandbook}.
For $\queryLanguage\in \{\text{AQ, IQ, CQ}\}$, the KBs $\K$ and $\K'$ are
\emph{\queryLanguage-inseparable}, in symbols $\K\equiv_\queryLanguage \K'$, if for every query $q\in\queryLanguage$,
we have that 
$\K\models q$ iff $\K'\models q$.

\mypar{Tree Representation and Homomorphisms}
We will also represent a concept expression $C$ as a finite directed tree $\treedef{C}$, where $\treeV{C}$ is the set of all vertices, with the root denoted by $\rho_C$, $\treeE{C}$ is the set of all edges, and $\treel{C}$ is a labelling function that maps every node to a set of concept names and every edge to a role name.
This \emph{tree representation} of $C$ uniquely represents the corresponding concept expression, and it is inductively defined  as follows:
	\begin{itemize}
		\item 
		for $C\!=\!\top$, 
		$\treeV{C}=\{\rho_C\}$ and
		$\treel{C}(\rho_C)=\emptyset$;
		\item 
		for $C\! =\! A$, where $A\in\NC$,  
		$\treeV{C}=\{\rho_C\}$ and $\treel{C}(\rho_C)= A$;
		\item 
		for $C\!=\!\exists r . D$,
		$\tree{C}$ is obtained from $\tree{D}$ by adding a new root $\rho_C$ and an
		edge from $\rho_C$ to the root $\rho_D$ of $\tree{D}$ with label $\treel{C}(\rho_C,\rho_D)=r$ (we
		call $\rho_D$ an \emph{$r$-successor} of $\rho_C$);
		\item
		for $C\!=\!D_1 \sqcap
		D_2$,   $\tree{C}$  is obtained by identifying the roots of $\tree{D_1}$  and
		$\tree{D_2}$, with $\treel{C}(\rho_C) = \treel{D_1}(\rho_{D_{1}})  \cup \treel{D_2}(\rho_{D_{2}}) $.
	\end{itemize}
The \emph{ABox representation of $C$}, $\Amc_C$, is the ABox encoding the tree representation $\tree{C}$ of $C$, defined as follows.
For each $v \in \treeV{C}$, we associate  $a_v \in \NI$.
Then, for every $u \in \treeV{C}$ and every $(u, v) \in \treeE{C}$, we put:
$A(a_u) \in \Amc_C$ iff $\treel{C}(u) = A$;
$r(a_u, a_v) \in \Amc_C$ iff $\treel{C}(u, v) = r$.
Given $\A, \A'$ be ABoxes,
a function $\homo \colon \indiv(\A) \to \indiv(\A')$ is called an \emph{ABox homomorphism}
from $\A$ to $\A'$
if:
for every $C(a) \in \A$, $C(\homo(a)) \in \A'$;
and, for every $r(a,b) \in \A$, $r(\homo(a), \homo(b)) \in \A'$.

\mypar{Learning Model}
We provide basic notions related to the exact learning model, extending the 
notation  in~\cite{KLOW18}.
A \emph{learning framework} $\Fmf$ is a quadruple $(\examples,\set, \hypothesisSpace, \mu)$, where $\examples$ 
is a set of \emph{examples}, \set is a subset of \examples, 
$\hypothesisSpace$ is a set of \emph{ concept representations} (also called \emph{hypothesis space}), and $\mu$ is a mapping from $\hypothesisSpace$
to $2^\examples$.   
We  omit \set if $\set=\examples $. 
Each element $l$ of $\hypothesisSpace$ is assumed to be represented using 
a set of symbols $\Sigma_l$ 
(if $l$ is a TBox $\Tmc$, $\Sigma_\Tmc$ is the signature of \Tmc).
We say that $\e\in \examples$ is a \emph{positive example} for $l \in \hypothesisSpace$ if $\e\in \mu(l)$ and a
\emph{negative example} for $l$ if $\e\not\in \mu(l)$.

Given a learning framework $\Fmf = (\examples, \set, \hypothesisSpace, \mu)$, we are interested in the
exact identification of a \emph{target}   concept representation $\target\in\hypothesisSpace$
w.r.t. the subset \set of examples,
by posing queries to oracles.
Let  ${\sf MQ}_{\Fmf,\target}$ be the oracle that takes as input some $\e \in \examples$ and
returns `yes' if $\e \in \mu(\target)$ and `no' otherwise. 
A \emph{membership query} is a call to the oracle ${\sf MQ}_{\Fmf,\target}$.
For every $\target \in \hypothesisSpace$, we denote by ${\sf EQ}_{\Fmf,\target}$ the oracle
that takes as input a \emph{hypothesis} concept representation $h \in \hypothesisSpace$
and returns `yes' if $\mu(h) \cap \set= \mu(\target)\cap \set$ and a \emph{counterexample} 
$\e \in (\mu(h) \oplus \mu(\target))\cap \set$ otherwise, where $\oplus$ denotes the symmetric
set difference.
There is no assumption regarding which counterexample in $(\mu(h) \oplus \mu(\target))\cap \set$
is chosen by the oracle.  
An \emph{equivalence query with respect to $\set$} is a call to the oracle ${\sf EQ}_{\Fmf,\target}$ (if $\set = \examples$, we omit `with respect to $\set$').

An 
\emph{(exact) learning algorithm} for $\Fmf=(\examples, \set,\hypothesisSpace, \mu)$ is a deterministic algorithm that, for a fixed but arbitrary $\target\in\hypothesisSpace$, 
takes $\Sigma_t$ as input, is allowed to make queries to ${\sf MQ}_{\Fmf,\target}$ and
${\sf EQ}_{\Fmf,\target}$ (without knowing what the target $\target$ to be learned
is), and that eventually halts and outputs some $h\in\hypothesisSpace$ with
$\mu(h) \cap \set= \mu(\target)\cap \set$.
We say that the learning algorithm is \emph{positive bounded} if, in addition,  $\mu(h)\subseteq \mu(t)$. 
We say that \Fmf is \emph{exact learnable} if there is a learning
algorithm for \Fmf and that \Fmf is \emph{polynomial query learnable}
if it is exact learnable by an algorithm $A$ such that at every step
the sum of the sizes of the inputs to membership and equivalence
queries made by $A$ up to that step is bounded by a polynomial
$p(|\target|,|\e|)$, where $\target$ is the target and $\e \in \set$ is the largest
counterexample seen so far \cite{arias2004exact}.  Similarly, \Fmf is
\emph{polynomial time learnable} if it is exact learnable by an
algorithm $A$ such that at every step (we count each call to an oracle
as one step of computation)
of computation the time used by $A$ up to
that step is bounded by a polynomial $p(|\target|,|\e|)$, where $\target\in \hypothesisSpace$
is the target and $\e \in \set$ is the largest counterexample seen so far.
We denote by \PQuery and \PTimeL 
the class of learning frameworks
which are, respectively, 
 polynomial query and polynomial time 
 learnable.
Clearly,  $\PTimeL\subseteq \PQuery$.

We now introduce the special case of learning frameworks which we focus in this work, called \emph{OMQA learning frameworks}.
Let \ontologyLanguage, \fixedAbox, and \queryLanguage be, respectively,
an ontology language, a fixed but arbitrary ABox, and a query language.
An OMQA learning framework $\Fmf(\ontologyLanguage,\fixedAbox,\queryLanguage)$ is a 
learning framework $(\examples,\set, \hypothesisSpace, \mu)$ where 
\examples is the set of all pairs $(\Amc,\query)$
with \Amc an ABox (may be different from \fixedAbox) and 
$\query \in\queryLanguage$; \hypothesisSpace
is the set of all TBoxes formulated in   \ontologyLanguage sharing 
a common finite signature (we write $\Sigma_\Tmc$ to refer 
to the signature of the target, which is assumed to be 
same as the one used for the hypothesis);
$\set$ is the set of elements $(\Amc,q)$ of \examples where 
$\Amc=\fixedAbox$;
and, for all $\targetOntology\in\hypothesisSpace$, $\mu(\targetOntology)=\{(\Amc,q)\in\examples\mid 
(\targetOntology,\Amc)\models q\}$.
We
assume that the signature of $q$, i.e., the set of concept and role names occurring in $q$, is $\Sigma_{\Tmc} \cup \Sigma_{\Amc}$.
Moreover, we define
the \emph{size} of an example $(\A, q)$, denoted by $|(\A, q)|$, as the sum of the size of $\A$ and $q$.
Given an OMQA learning framework 
$\Fmf(\ontologyLanguage,\fixedAbox,\queryLanguage)=(\examples, \set,\hypothesisSpace, \mu)$, 
for all $\hypothesisOntology,\targetOntology\in \hypothesisSpace$,
 if $\mu(\hypothesisOntology)\cap \set=\mu(\targetOntology)\cap \set$,  then  
 $\hypothesisOntology,\targetOntology$ are \queryLanguage-inseparable w.r.t. $\fixedAbox$. 
Since an equivalence query in an OMQA learning framework with $\set \neq \examples$ is in fact asking 
whether a given TBox is \emph{query inseparable} from the target TBox w.r.t. 
a fixed ABox and a query language, 
we may call it an \emph{inseparability query}. 
 
\section{Polynomial Learnability}\label{sec:learning-algorithms}
We investigate whether, for a given  fixed ABox and 
 query language, the problem of learning query inseparable  \ELH TBoxes is polynomial.
We first discuss the relationship between our OMQA setting and the data retrieval
one~\cite{DBLP:conf/aaai/KonevOW16}.
The  difference between the two settings 
is that here the oracle can only choose counterexamples of the form  
$(\fixedAbox,q)$, with the fixed
ABox
\fixedAbox given as input, 
whereas in the mentioned work the ABox in a counterexample can be  arbitrary.
In both settings the learner can pose membership queries with an arbitrary 
ABox in the examples.
 
Formally, given an ontology language \ontologyLanguage and 
a query language \queryLanguage, we denote by $\Fmf(\ontologyLanguage,\queryLanguage)$
the learning framework $(\examples, \set,\hypothesisSpace, \mu)$, 
where \examples is the set of all pairs $(\Amc,\query)$
with \Amc an ABox  and 
$\query \in\queryLanguage$; \hypothesisSpace
is the set of all TBoxes formulated in   \ontologyLanguage; 
$\examples=\set$; and, for all $\targetOntology\in\hypothesisSpace$, $\mu(\targetOntology)=\{(\Amc,q)\in\examples\mid 
(\targetOntology,\Amc)\models q\}$. 
We denote by $\ELH_{\sf lhs}$ and $\ELH_{\sf rhs}$ the fragments of 
\ELH which only allow complex concept expressions on the left-hand side and 
on the right-hand side of CIs, respectively. 
It is known that the learning frameworks 
$\Fmf(\ELH_{\sf lhs},\AQ)$ and $\Fmf(\ELH_{\sf rhs},\IQ)$ 
are in \PTimeL, whereas 
$\Fmf(\ELH,\IQ)$ and $\Fmf(\ELH_{\sf rhs},\CQ)$ are not in \PQuery~\cite{DBLP:conf/aaai/KonevOW16}.
It follows from our definitions that, for any ontology language \ontologyLanguage, query language \queryLanguage 
and ABox \fixedAbox, 
polynomial learnability of $\Fmf(\ontologyLanguage,\queryLanguage)$ implies 
polynomial learnability of $\Fmf(\ontologyLanguage,\fixedAbox,\queryLanguage)$. Thus, 
the positive results for $\Fmf(\ELH_{\sf lhs},\AQ)$ and $\Fmf(\ELH_{\sf rhs},\IQ)$  
hold in the OMQA setting. That is, 
for any ABox $\fixedAbox$, 
the learning frameworks $\Fmf(\ELH_{\sf lhs},\fixedAbox,\AQ)$ and $\Fmf(\ELH_{\sf rhs},\fixedAbox,\IQ)$   
are in \PTimeL. 
\begin{proposition}\label{thm:relationshipbetweenmodels}
For any ontology language \ontologyLanguage, query language \queryLanguage 
and ABox \fixedAbox, 
if $\Fmf(\ontologyLanguage,\queryLanguage)$ is in \PTimeL/\PQuery, then 
$\Fmf(\ontologyLanguage,\fixedAbox,\queryLanguage)$ is in \PTimeL/\PQuery.
\end{proposition}
In the following, we extend the positive results for $\ELH_{\sf lhs}$ and $\ELH_{\sf rhs}$
 to the class of 
\ELH terminologies (the union of $\ELH_{\sf lhs}$ and $\ELH_{\sf rhs}$) 
 in the OMQA setting based on \IQs.
This is in contrast with the negative result for the class of \ELH terminologies in the data retrieval setting 
with \IQs, 
thus, showing that the converse direction of Proposition~\ref{thm:relationshipbetweenmodels} 
does not hold. 
Throughout this section $\fixedAbox$ is a fixed but arbitrary ABox. 
We also analyse the learning framework $\Fmf(\ELH_{\sf rhs},\fixedAbox,\CQ)$, which 
is also not in \PQuery in the OMQA setting.

\mypar{Learning \ELH ontologies with \AQ}
We argue that the learning framework $\Fmf(\ELH,\fixedAbox,\AQ)$ is  in \PTimeL.
There are only polynomially many counterexamples that the oracle can give 
since they can only be of the form $(\fixedAbox,q)$ with 
$q$ an atomic query (using symbols from $\Sigma_\targetOntology$).
The set of RIs entailed by \targetOntology can 
be learned in polynomial time 
by adding to the hypothesis \hypothesisOntology 
all RIs $r\sqsubseteq s$ such that 
the membership oracle 
 replies `yes' given an example $(\{r(a,b)\},s(a,b))$
 as input. 
Therefore we only need to show that one can compute 
concept inclusions that together with  \fixedAbox entail precisely 
the same atomic queries as the target ontology \targetOntology 
on  \fixedAbox. 
The following lemma establishes that indeed 
one can compute such concept inclusions in polynomial time.
		
\begin{restatable}{lemma}{learningAQlemma}\label{lem:polyexample}
Let  \targetOntology and \hypothesisOntology be resp. the target and the hypothesis 
(of size polynomial in $|\targetOntology|$)
terminologies. 
Given a positive counterexample $(\fixedAbox,A(a))$, 
one can compute in polynomial time in $|\fixedAbox|$ and $|\targetOntology|$
a CI $C\sqsubseteq B$ such that 
  $(\targetOntology,\fixedAbox)\models B(b)$, $(\hypothesisOntology,\fixedAbox)\not\models B(b)$, and
$(\hypothesisOntology\cup\{C\sqsubseteq B\},\fixedAbox)\models B(b)$, for some 
$b\in\individuals{\fixedAbox}$. Moreover, $\T\models C\sqsubseteq B$.
\end{restatable}
The main idea for proving Lemma~\ref{lem:polyexample} is 
to 
transform the structure of \fixedAbox into a tree shaped structure, as in~\cite{DBLP:conf/aaai/KonevOW16}. 
However, in the mentioned work only CIs of the form $C\sqsubseteq A$ are considered, 
while in our case CIs of the form 
$A\sqsubseteq C$ may also be present. 
With such CIs 
the computed tree shaped ABox 
may be smaller than the original concept expression in $\targetOntology$ (as in Example~\ref{ex:sizeabox}).
\begin{example}\label{ex:sizeabox}\upshape
Assume that $\T = \{B\sqsubseteq \exists s.B, \exists r.\exists s.B\sqsubseteq A\}$ 
and $\A=\{r(a,b),B(b)\}$. 
We have that 
$(\T,\A)\models A(a)$ and
the concept expression $\exists r.B$ encoded in \A is smaller than 
the original concept expression $\exists r.\exists s.B$ in \T implying the concept name $A$. 
Intuitively, even though there is no homomorphism from 
$\A_{\exists r.\exists s.B}$ to \A, we have that $B$ `abbreviates' 
$ \exists s.B$, because it is implied by $B$. 
\hfill {\mbox{$\triangleleft$}}
\end{example}
Since there are polynomially many possible counterexamples, 
our upper bound for \AQs 
follows from Lemma~\ref{lem:polyexample}.

\begin{restatable}{theorem}{learningAQ}\label{thm:aq}
	$\Fmf(\ELH,\fixedAbox,\AQ)$ is  in \PTimeL. 
	Moreover, there is a positive bounded learning algorithm for showing such upper 
	bound which only poses membership queries (no inseparability queries are needed). 
\end{restatable}

\mypar{Learning \ELH ontologies with \IQ}
We build on the result for \AQs (Theorem~\ref{thm:aq}) and argue that the learning framework $\Fmf(\ELH,\fixedAbox,\IQ)$ is  in \PTimeL.
By Theorem~\ref{thm:aq},   CIs of the form $ C\sqsubseteq A $ which ensure 
\AQ-inseparability w.r.t. \fixedAbox can be learned with 
membership queries. 
 It remains to show how one can learn CIs of the form $ A\sqsubseteq C $, 
 so that \hypothesisOntology is \IQ-inseparable from \T (w.r.t. \fixedAbox). 
 By Theorem~\ref{thm:aq}, we can assume that one can construct a hypothesis \Hmc
 such that  $\T\models\Hmc$, since there is a positive bounded learning algorithm.  
 The next lemma states that if $\T\models\Hmc$ and $ (\Hmc,\fixedAbox) \equiv_{\AQ} (\Tmc,\fixedAbox)$ 
 then one can transform a counterexample of the form $(\fixedAbox,D(a))$ into 
 a CI such that $\Tmc\models A\sqsubseteq C$, and $\Hmc\not\models A\sqsubseteq C$, with 
 $C$ a subconcept of $D$ and $A \in \Sigma_\T$ such that $(\hypothesisOntology,\fixedAbox)\models A(b)$ for 
 some $b\in\NI$. Since $|\Sigma_\T|$ is polynomial in $|\Tmc|$ and the number of subconcepts of $D$ is also polynomial 
 in $|D|$. One can find such  $A\sqsubseteq C$ in polynomial time w.r.t. $|\T|$ and $|(\fixedAbox,D(a))|$. 
 \begin{lemma}\label{lem:iq}
 Let \T and \Hmc be \ELH terminologies and let $\fixedAbox$ be an ABox. 
 Assume $ (\Hmc,\fixedAbox) \equiv_{\AQ} (\Tmc,\fixedAbox)$.  
  If $(\fixedAbox,D(a))\in \mu(\T)\setminus\mu(\Hmc)$ then 
  there is $A(b)$ 
  and a subconcept $C$ of $D$ 
  such that $(\hypothesisOntology,\fixedAbox)\models A(b)$, $\Tmc\models A\sqsubseteq C$, and $\Hmc\not\models A\sqsubseteq C$.
 \end{lemma}
 Given $A\sqsubseteq C$ such that 
 $\Tmc\models A\sqsubseteq C$, and $\Hmc\not\models A\sqsubseteq C$, 
 one can compute with polynomially many polynomial size queries
  another CI $A'\sqsubseteq C'$ entailed by \T but not by \Hmc 
  belonging to a class of CIs called \emph{\T-essential}~\cite[Lemma 29]{KLOW18}.
   In fact this can be done 
    in polynomial time given the complexity of entailment checking~\cite{BBL-EL} (no inverse roles). 
  Such \T-essential CIs have the property that their size is bounded polynomially 
  in the size of \T~\cite[Lemma 32]{KLOW18} and if $\alpha_1= A'\sqsubseteq C_1$ and $\alpha_2=A'\sqsubseteq C_2$ are 
  \T-essential and  not equivalent then  
one can compute in polynomial time a \T-essential CI $A'\sqsubseteq C'$ such that 
it entails $\alpha_1$ and $\alpha_2$~\cite[Lemma 30]{KLOW18}. All in all, 
if the learner computes such \T-essential counterexamples and adds/refines 
 them in the hypothesis (see~\cite[Algorithm~2]{KLOW18}) 
 then, after learning from polynomially many counterexamples, it will terminate and output
  a hypothesis \IQ-inseparable from the target (w.r.t. \fixedAbox). 
 The presence of CIs of the form $C\sqsubseteq A$ does not affect 
 this result~\cite{DBLP:conf/kr/DuarteKO18}. 
 \begin{restatable}{theorem}{learningIQ}\label{thm:iqupper}
	$\Fmf(\ELH,\fixedAbox,\IQ)$ is in  \PTimeL.
\end{restatable}

In contrast, the learning framework $\Fmf(\ELH,\IQ)$ is not in 
\PQuery~\cite{DBLP:conf/aaai/KonevOW16}.  
We observe that the counterexamples used in such hardness proof 
are based on an exponential number of ABoxes encoding concept expressions of the form $C_{\mathbf{b}}=\sqcap_{i\leq n} C_i$, 
where $\mathbf{b}=b_1 \ldots b_n$ is a sequence with $b_i\in\{0,1\}$,
$C_i=A_i$ if $b_i=1$, and $C_i=B_i$ if  $b_i=0$.
In the OMQA setting, the ABox is fixed and, as stated in Theorem~\ref{thm:iqupper}, this lowers the complexity.

\mypar{Learning \ELH ontologies with \CQ}
  $\ELH$ ontologies  are not polynomial  query learnable in the data retrieval 
  setting with \CQs as the query language (in fact not even 
   the fragment $\ELH_{\sf rhs}$)~\cite{DBLP:conf/aaai/KonevOW16}. 
  The counterexamples used 
  by the oracle in the hardness proof 
  are of the form $(\{A(a)\},q)$~\cite[proof of Lemma~8]{DBLP:conf/aaai/KonevOW16}, 
  so  $\{A(a)\}$ can be considered 
  as the fixed ABox given as part of the input in an OMQA learning framework. 
  Thus, the mentioned hardness result can be transferred to our setting. 
   We formalise this result with the next theorem. 
   \begin{theorem}\label{thm:elhnotpquery}
   $\Fmf(\ELH,\fixedAbox,\CQ)$ 
   is not in  \PQuery.
   \end{theorem}

The hardness proof in the mentioned paper uses 
a very simple \CQ of the form $\exists x M(x)$, which has a match in  
the anonymous part of the model but `hides' the concept 
on the right side of a CI that causes the entailment of this query. 
This phenomenon makes one wonder whether restricting to the class 
of queries in which every variable needs to be reachable by an individual name (as it happens with \IQs) 
can tame the complexity of the problem. 
Our next theorem proves this.

Given a CQ $q$, we define $G_q$ as the directed graph $(V,E)$
where the nodes $V$ are the terms of $q$ and the edges $E$
are the pairs $(t_1,t_2)$ such that there is an atom of the form 
$r(t_1,t_2)$ in $q$. We say that a CQ $q=\exists \vec{x}\varphi(\vec{a},\vec{x})$ is \emph{rooted} if 
for every $x$ in $\vec{x}$, we have that $x$ is reachable from 
a node in $G_q$ that is 
in $\vec{a}$. 
We denote by $\CQr$ the class of all rooted CQs. 
The next lemma establishes that one can transform queries in $\CQr$
into queries in \IQ (by posing membership queries).

\begin{restatable}{lemma}{lemmalearningCQr}\label{lemma:cqrupper}
	 Let \T and \hypothesisOntology  be \ELH TBoxes and assume \T and \hypothesisOntology entail the same RIs.
	 Given a positive counterexample $ (\fixedAbox,q) $ (for 
	  \T and \Hmc), with $q\in\CQr$, 
	 one can contruct a positive counterexample $ (\fixedAbox,q') $ with 
	 $q'\in\IQ$ in polynomial time 
	 in  $|(\fixedAbox,q)||\Sigma_\T|$. 
\end{restatable}
In Figure~\ref{fig:cqrexample}, it is shown an example of this conversion.
\begin{figure}
	\centering
	\begin{tikzpicture}
	\tikzset{every node/.style={circle,fill=black,inner sep=0pt,minimum size=4pt}}
	\tikzset{every edge/.append style={every node/.style={circle,fill=white,inner sep=0pt,minimum size=0pt}, >=stealth}}

	\node (1) at (0,0) {} ;
	\node[fill=white,opacity=0,text opacity=1] () at (0,0.2) { $ a  $} ;
	\node (x1) at (1,0.3) {} ;
	\node[fill=white,opacity=0,text opacity=1] () at (1,0.5) { $ x_1  $} ;
	\node (x2) at (1,-0.3) {} ;
	\node[fill=white,opacity=0,text opacity=1] () at (1,-0.5) { $ x_2  $} ;
	\node (x3) at (2,0.3) {} ;
	\node[fill=white,opacity=0,text opacity=1] () at (2,0.5) { $ x_3  $} ;
	\node (x4) at (2,-0.2) {} ;
	\node[fill=white,opacity=0,text opacity=1] () at (2,0) { $ x_4  $} ;
	\node (x5) at (2,-0.7) {} ;
	\node[fill=white,opacity=0,text opacity=1] () at (2,-0.5) { $ x_5  $} ;

	\draw[->, >=stealth] (1) to (x1) ;
	\node[fill=white,opacity=0,text opacity=1] () at (0.5,0.3) { $ r  $} ;
	\draw[->, >=stealth] (1) to (x2) ;
	\node[fill=white,opacity=0,text opacity=1] () at (0.5,-0.3) { $ r  $} ;
	\draw[->, >=stealth] (x1) to (x3) ;
	\node[fill=white,opacity=0,text opacity=1] () at (1.55,0.45) { $ s  $} ;
	\draw[->, >=stealth] (x1) to (x4) ;
	\node[fill=white,opacity=0,text opacity=1] () at (1.55,0.175) { $ s  $} ;
	\draw[->, >=stealth] (x2) to (x4) ;
	\node[fill=white,opacity=0,text opacity=1] () at (1.55,-0.35) { $ s  $} ;
	\draw[->, >=stealth] (x2) to (x5) ;
	\node[fill=white,opacity=0,text opacity=1] () at (1.55,-0.65) { $ s  $} ;
	
	\node[fill=white,opacity=0,text opacity=1] () at (3,0) { $ \rightarrow  $} ;
	\node (a) at (4,0) {} ;
	\node[fill=white,opacity=0,text opacity=1] () at (4,-0.2) { $ a  $} ;
	\node (a1) at (5,0) {} ;
	\node[fill=white,opacity=0,text opacity=1] () at (5,-0.2) { $ x_1  $} ;
	\node (a2) at (6,0) {} ;
	\node[fill=white,opacity=0,text opacity=1] () at (6,-0.2) { $ x_3  $} ;

	\draw[->, >=stealth] (a) to (a1) ;
	\node[fill=white,opacity=0,text opacity=1] () at (4.5,0.15) { $ r  $} ;
	\draw[->, >=stealth] (a1) to (a2) ;
	\node[fill=white,opacity=0,text opacity=1] () at (5.5,0.15) { $ s  $} ;
	\end{tikzpicture}
	\caption{
	Assume  $ \T = \{A \sqsubseteq \exists r.\exists s.\top \} $, $ \fixedAbox = \{A(a) \} $ 
	and $ \hypothesisOntology = \emptyset $. A call to ${\sf EQ}_{\Fmf,\T}$ can output 
		$ (\fixedAbox, \exists \vec{x} (r(a,x_1) \land r(a,x_2) \land s(x_1,x_3) \land s(x_1,x_4)\land s(x_2,x_4) \land s(x_2,x_5) )  $,  
		which can be converted into $ (\fixedAbox, \exists r.\exists s.\top(a)) $ by merging variables 
		and asking ${\sf MQ}_{\Fmf,\T}$ whether the new query holds.  
	 }
	\label{fig:cqrexample}
\end{figure}
Even though the conversion involves 
deciding query answering, which is \NP-hard, these checks are on the `side' of the oracle, and so, they 
do not affect the complexity of learning.
By Lemma~\ref{lemma:cqrupper} a $ \IQ $-inseparable hypothesis can be found 
by following the same steps of a learning algorithm for  \IQ   after $q$ in  $ (\fixedAbox,q) $ is converted 
into an instance query. For \ELH, if TBoxes 
 entail the same RIs then 
  $\IQ$-inseparability implies $\CQr$-inseparability. 
  Since RIs can be easily learned with membership queries, we obtain our next theorem.

\begin{restatable}{theorem}{learningCQr}\label{thm:cqrupper}

   $\Fmf(\ELH,\fixedAbox,\CQr)$ 
   is   in  \PTimeL. 
\end{restatable}

\section{Data Updates}\label{sec:data-updates}
The algorithm presented in the previous section for \IQs 
computes an ontology \hypothesisOntology that is
\IQ-inseparable
from the target
\targetOntology w.r.t. a fixed ABox \fixedAbox.
In this section, we
first
study
when
$\IQ$-inseparability is preserved,
without
changes to
the hypothesis \hypothesisOntology,
if \fixedAbox is updated to an ABox $\Amc$.
Then, 
given an OMQA learning framework
$\Fmf(\ELH,\fixedAbox,\IQ) = (\examples,\set , \hypothesisSpace, \mu)$,
we determine conditions
on
an updated ABox $\Amc$,
sufficient to guarantee that a learning framework
$(\examples, \set', \hypothesisSpace, \mu)$ with $\set'\supseteq \set$ is still in $\PTimeL$.

To characterise when
$ \IQ $-inseparability
is preserved if
\fixedAbox is updated to an ABox \A, we use the classical 
notion of bisimulation. 
Let $\I=(\Delta^\I,\cdot^\I),\Jmc=(\Delta^\Jmc,\cdot^\Jmc)$ be two interpretations.
A \emph{bisimulation} is a non-empty 
relation
$\Zmc \subseteq \Delta^\I \times \Delta^\Jmc$
satisfying the following conditions, for all $(d,e) \in \Zmc$: 
(1) for all concept names $A \in \NC$, $d \in A^\I $ iff $e \in A^\Jmc$;
(2) for all role names $r \in \NR$, 
if $(d,d') \in r^\I, d' \in \Delta^\I$, then there exists $e' \in \Delta^\Jmc$ such that $(e,e') \in r^\Jmc$ and $(d',e') \in \Zmc$;
(3) for all role names $r \in \NR$, if $(e,e') \in r^\Jmc, e' \in \Delta^\Jmc$, then there exists $d' \in \Delta^\I$ such that $(d,d') \in r^\I$ and $(d',e') \in \Zmc$.
If $(d,e)\in\Zmc$,  we write $(\I,d) \sim (\Jmc,e)$. 

\begin{restatable}{theorem}{updateinseparabilitybisimulation}\label{lem:update1} 

Let \T and \Hmc be \ELH terminologies
entailing the same RIs,
and let $\fixedAbox$ and $\A$ be ABoxes.
If, for all $b \in \individuals{\A}$, there is $a \in \individuals{\fixedAbox}$ such that
$(\I_{\fixedAbox},a) \sim (\I_{\A},b)$,
then
$(\Hmc,\fixedAbox) \equiv_{\IQ} (\Tmc,\fixedAbox)$
implies
$(\hypothesisOntology, \A) \equiv_{\IQ} (\T, \A)$.

\end{restatable}

Theorem~\ref{lem:update1}  does not hold if we require the ABoxes $\fixedAbox$ and $\A$ 
to be \emph{homomorphically equivalent}, i.e., 
if there are ABox homomorphisms from $\fixedAbox$ to $\A$ and from $\A$ to $\fixedAbox$.
\begin{example}\label{ex:genhyp}
Consider
$ \T = \{ 
\exists r.A_1 \sqsubseteq B \} $
and
$ \fixedAbox = \{r(a,b),A_1(b),A_2(b) \} $.
The 
hypothesis 
$ \hypothesisOntology=\{ 
\exists r.(A_1\sqcap A_2) \sqsubseteq B \} $ is \IQ-inseparable.
However, if  
$\A = \fixedAbox \cup \{r(a',b'),A_1(b') \} $,
then $\IQ$-inseparability is not preserved, 
even though $\fixedAbox$ and
$\A$
are homomorphically equivalent.
 \hfill {\mbox{$\triangleleft$}}
\end{example}

The problem here is that the left-hand side of CIs in $\hypothesisOntology$ could be biased and too specific for the individuals in $\fixedAbox$.
Indeed, in Example~\ref{ex:genhyp}, if the more general concept expression $\exists r.A_1$ on the left-side 
had been learned, then \IQ-inseparability would have been preserved after the update. 
If we
allow for modifications
to the
learned hypothesis \hypothesisOntology, 
we can extend the class of updated ABoxes \A not only to those in which every 
individual in \A is bisimilar to an individual in \fixedAbox but  in which a  more relaxed condition is required.
It is easy for the learner to make certain kinds of generalisation,  for instance,
check whether $ \T\models  \exists r.A_1 \sqsubseteq B$ and
add such more general CI to \hypothesisOntology.
Therefore, the idea is to suitably `generalise' the left-hand side of CIs in the hypothesis $\Hmc$
computed by the learning algorithm.

\emph{Generalisation of $C \sqsubseteq A \in  \hypothesisOntology$ for \T} consists of 
replacing $ C $ by the result $ C' $ of  
(1) 
replacing a concept name $ B $ in $ C $ with $ \top $ or $ B' $
such that $ \T\models B \sqsubseteq B' $ and $ \T\not\models B' \sqsubseteq B $  
if 
$ \T\models C' \sqsubseteq A $;
or
(2) 
replacing a role name $ r $ in $ C $ with $ s $ 
such that $ \T\models r \sqsubseteq s $ and $ \T\not\models s \sqsubseteq r $ 
if $ \T\models C' \sqsubseteq A $.
We say that $ C\sqsubseteq A\in\hypothesisOntology $ is \emph{generalised} for \T
if generalization of $C \sqsubseteq A \in  \hypothesisOntology$ for \T has been exhaustively applied.
  \hypothesisOntology is generalised for \T if all the CIs in it are generalised.
  We may omit `for \T' if this is clear from the context.

With
the following definitions, we define a class of ABoxes that 
are guaranteed to  preserve 
$ \IQ $-inseparability if the hypothesis is generalised.  
Given a TBox \T
and  concept names $A,B \in \Sigma_{\T} \cap \NC$, 
we say that there is a \emph{linear derivation} from $A$ to $B$ if 
$\T\models A\sqsubseteq B$ and for all $B'\in\NC$ such that $\T\models A\sqsubseteq B'$
we have that $\T\models B'\sqsubseteq B$. Similarly, for $r,s \in \Sigma_{\T} \cap \NR$,
there is a \emph{linear derivation} from $r$ to $s$ if 
$\T\models r\sqsubseteq s$ and for all $s'\in\NR$ such that $\T\models r\sqsubseteq s'$
we have that $\T\models s'\sqsubseteq s$. 
We write $\Amc < \Amc'$ if $\Amc'$ is the result of replacing $A(a)\in\Amc$ by 
$B(a)$ and there is a linear derivation from $A$ to $B$, or if $\Amc'$ is the result of replacing $r(a,b)\in\Amc$ by 
$s(a,b)$ and there is a linear derivation from $r$ to $s$. 
We define $ {\sf g}_{\T}(\fixedAbox) $ as the set of all 
ABoxes $\Amc$ such that there is a sequence $\Amc_1 < \ldots < \Amc_n$ 
with $\Amc_1=\fixedAbox$ and $\Amc_n=\Amc$.

The following theorem establishes an  our upper bound 
for
learning frameworks
extending $\set$ with all the examples of the form $(\A, q)$, where $\A \in {\sf g}_\T(\fixedAbox)$ and $q\in\IQ$.

\begin{restatable}{theorem}{learningpoltimegeneralize}\label{thm:update} 
	\label{lgeneraliseinsep}
	Let \Fmf be the learning framework that results from 
	adding all pairs of the form $(\Amc,q) $, with $\A \in {\sf g}_\T(\fixedAbox)$
	and
	$q \in \IQ$,
	to the set $\set$ in $\Fmf(\ELH,\fixedAbox,\IQ)=(\examples,\set , \hypothesisSpace, \mu)$, where $\T\in\hypothesisSpace$. 
	Assume $\Sigma_\T\subseteq \Sigma_{\fixedAbox}$.
	Then, 
	\Fmf is in  \PTimeL.
\end{restatable}
 
\section{Learning from Data}\label{sec:learning-data}

The existence of oracles that correctly answer to
all
the
queries posed by the learner
does not naturally fit
those
settings in which only a direct access to data is available.
In this section, we study how the oracle-based approach presented so far can 
be modified so to allow access to examples retrieved from data, thus reducing 
the dependency of our learning model on membership and inseparability queries.
Firstly, we consider
a finite 
batch of examples~\cite{ariasetal}, to be used as a representative of 
the entire data pool, and study conditions under which it is guaranteed 
the existence of such a batch that allows us to learn inseparable ontologies.
Then, we analyse how a data-driven approach can be used as a basis for a learning model for DL ontologies 
based on the well-known PAC learning model, possibly extended with membership queries~\cite{Valiant}.

\mypar{Learning from batch}
Given an OMQA learning framework $\Fmf(\ontologyLanguage,\fixedAbox,\queryLanguage) = (\examples,\set, \hypothesisSpace, \mu)$,
  a \emph{batch} $\batch$
is a \emph{finite} subset of $\examples$.
One could ask under which conditions a batch 
 that allows us to construct an ontology $\hypothesisOntology \in \hypothesisSpace$ 
 which is $\queryLanguage$-inseparable from a target $\targetOntology \in \hypothesisSpace$ 
 is guaranteed to exist.
If no restrictions 
 are imposed on the form of the examples occurring in $\batch$, the answer is trivial for 
 $\EL$ ontologies and \IQs.
Indeed, for every $C \sqsubseteq D \in \targetOntology$,
consider the set $\batch$ of examples of the form $(\Amc_{C}, D(\rho_{C}))$,
obtained by representing the concept $C$ as a labelled tree with root $\rho_{C}$ and encoded in the ABox $\Amc_{C}$.
These examples 
have the property that, for 
every $\targetOntology \in \hypothesisSpace$, $\Tmc \models C \sqsubseteq D$ iff $(\targetOntology, \Amc_{C}) \models D(\rho_{C})$.
By setting $$\hypothesisOntology = \{ C \sqsubseteq D \mid (\Amc_{C}, D(\rho_{C})) \in \batch \},$$ we 
obtain that \hypothesisOntology 
is equivalent to (and thus $\IQ$-inseparable, w.r.t. any ABox, from) \targetOntology. 
This construction can be easily 
extended to   \ELH by using examples of the form $(\{r(a,b)\},s(a,b))$.
However, instead of allowing the examples retrieved from our data to have no restrictions in their size and 
in the shape of the ABoxes, it would be more natural to require that these ABoxes
contain less information than $\fixedAbox$, given as a parameter.
This intuition can be made precise by imposing that, for each ABox $\Amc$ in the batch, there is an ABox homomorphism from $\Amc$ to $\fixedAbox$ and 
$|\Amc|$ is polynomial in $|\T|$. 
Our next theorem states that, under these assumptions, one can construct  a $\queryLanguage$-inseparable \ELH terminology. 

\begin{restatable}{theorem}{existsbatch}\label{existsbatch}
Let $\Fmf(\ELH, \fixedAbox, \queryLanguage) = (\examples, \set, \hypothesisSpace, \mu)$ 
be an OMQA learning framework, with~$\queryLanguage \in \{ \AQ, \IQ,\CQr \}$,
and let $\targetOntology \in \hypothesisSpace$ be such that $\Sigma_{\targetOntology} \subseteq \Sigma_{\fixedAbox}$.
Let $\homsubset \subseteq \examples$ be the set of examples $(\Amc, q)$
such that there is
an ABox homomorphism
from $\Amc$ to $\fixedAbox$.
Then, there is a batch $\batch \subseteq \homsubset$,
polynomial in 
$|\targetOntology|$,
and an
algorithm
such that it 
takes $\batch$ as input, it eventually halts, and returns
$\hypothesisOntology \in \hypothesisSpace$ such that
$\mu(\hypothesisOntology) \cap \set = \mu(\targetOntology) \cap \set$.
\end{restatable}

\mypar{PAC learning}
Let
$\Fmf = (\examples,\set, \hypothesisSpace, \mu)$ be
a learning framework. A \emph{probability distribution} $\prob$ on
$\set$
is a function
$\prob \colon 2^{\set} \to [0, 1] \subset \mathbb{R}$
such that 
$\prob(\bigcup_{i \in I} X_{i}) = \sum_{i \in I} \prob(X_{i})$ for mutually disjoint $X_i$,
where $I$ is a countable set of indices, $X_{i} \subseteq \set$, and
$\prob(\set) = 1$.
Given a target $\target \in \hypothesisSpace$,
let $\EX$ be the oracle that takes no input, and outputs a 
\emph{classified example}
$(e,\lab{e}{t})$, where $e \in \set$ is
sampled according to the probability distribution $\prob$, 
  $\lab{e}{t}=1$, if $e \in \mu(\target) \cap \set$ (\emph{positive example}), and $\lab{e}{t}=0$, otherwise (\emph{negative example}). 
An \emph{example query} is a call to the oracle $\EX$.
A 
\emph{sample} generated by $\EX$ is a (multi-)set of indexed classified examples,
independently and identically 
distributed
according to $\prob$, sampled by calling $\EX$.
A learning framework
$\Fmf$
is \emph{PAC learnable}
if there is a function $f : (0, 1)^{2} \to \mathbb{N}$
and a deterministic algorithm
such that, for every $\epsilon, \delta \in (0, 1) \subset \mathbb{R}$, every probability distribution $\prob$ on $\set$,
and every target
$\target \in \hypothesisSpace$,
given a sample 
of size $m \geq f(\epsilon, \delta)$ generated by $\EX$,
the algorithm
always halts and outputs $h \in \hypothesisSpace$ such that
with probability at least $(1 - \delta)$ over the choice of $m$ examples in $\set$,
we have that $\prob((\mu(h) \oplus \mu(t)) \cap \set) \leq \epsilon$.
If
the time used by
the algorithm
is bounded by a polynomial function $\poly(|\target|, |e|, 1/\epsilon, 1/\delta)$,
where
$e$ is the largest example in the sample, 
then we say that $\Fmf$
is \emph{polynomial time PAC learnable}.
If, in addition, the algorithm is allowed to make
membership queries (where each call to ${\sf MQ}_{\Fmf,\target}$ counts as one step of computation),
we say that $\Fmf$ is \emph{polynomial time PAC learnable with membership queries}.

\begin{restatable}[\cite{angluinqueries},~\cite{MohEtAl}]{theorem}{paclearn}\label{thm:paclearn}
If
$\Fmf$ 
is
in $\PTimeL$,
then $\Fmf$ 
is polynomial
time 
PAC learnable with membership queries.
\end{restatable}

However, the converse direction of Theorem~\ref{thm:paclearn} does not hold~\cite{Blum:1994:SDM:196751.196815}. 
The argument in the mentioned paper is based on the assumption that one-way functions exist and cannot 
be easily adapted to serve as a counterexample for OMQA learning frameworks.
Our next result,
showing that the converse direction of Theorem~\ref{thm:paclearn} does not hold in our setting,
does \emph{not} rely on cryptographic assumptions, however, it is representation-dependent. 
Given a sequence ${\ssigma} = \sigma_1\sigma_2 \ldots \sigma_n$, with 
$\sigma_i \in \{r,s\}$, let the expression $\exists \ssigma.C$ stand for 
$\exists \sigma_1.\exists \sigma_2. \ldots .\exists \sigma_n.C$ (clearly, there are $2^n$ expressions of this form).
Consider the OMQA learning framework $\Fmf(L,\fixedAbox,Q)$ 
where  $\fixedAbox=\{ A(a) \}$; 
$Q$ is the query language that extends IQs with a CQ of 
the form $\exists x M(x)$;
and  $L$  is an ontology language allowing only $\ELH_{\sf rhs}$ TBoxes of the form $\Tmc_{\ssigma} = \{ A \sqsubseteq \exists \ssigma.M \} \cup \Tmc_{0}$
to be expressed, 
where
$$\Tmc_{0} = \{ A \sqsubseteq X_{0}, M \sqsubseteq \exists r. M \sqcap \exists s. M \} \ \cup$$
$$\{ X_{i} \sqsubseteq  \exists r. X_{i + 1} \sqcap  \exists s. X_{i +1} \mid 0 \leq i < n \}.$$
It can be shown, with an argument similar to the one used in~\cite[proof of Lemma~8]{DBLP:conf/aaai/KonevOW16} 
(and Theorem~\ref{thm:elhnotpquery} above),
that such framework is not polynomial query learnable.
However, due to the restrictions on the hypothesis space, 
it is polynomial
time 
PAC learnable, even without membership queries.

\begin{restatable}{theorem}{pacnotimplyexact}\label{thm:pacnotimplyexact}

There is a polynomial time
PAC learnable OMQA learning framework that is not
in $\PQuery$.

\end{restatable}

A learning framework $\Fmf = (\examples, \set, \hypothesisSpace, \mu)$ 
\emph{shatters} a set of examples $\setexamples\subseteq \set$ if 
$|\{\mu(h)\cap \setexamples\mid h\in\hypothesisSpace\}|=2^{|\setexamples|}$. 
The \emph{VC-dimension}~\cite{Vapnik:1995:NSL:211359} of
\Fmf, denoted $\vc{\Fmf}$, is the 
maximal size of a set $\setexamples\subseteq \set$ 
such that $\Fmf$ shatters $\setexamples$. If \Fmf can shatter arbitrarily large sets then
\Fmf has infinite VC-dimension.
\begin{example}\upshape
For  $ n\in \Nbb $, let 
$$\fixedAbox^n = \{ r(a_i,a_{i+1}),s(a_i,a_{i}) \mid 1\leq i < n  \} \cup \{r(a_n,a_1)\} .$$ 
Each
$ a_i $ can be identified by
$C_i= \exists r^{n-i}.\exists s.\top $,
since
$ a_i\not\in C^{\I_{\fixedAbox^n}}_i $ ($\exists r^k$ is a shorthand 
for $k$ nestings of the form $\exists r$).
E.g., $ a_1$ is the only individual not in $ (\exists r.\exists s.\top)^{\I_{\fixedAbox^2}}$ (see Figure~\ref{fig:vcdimension}). 
 For all $n\in\Nbb$, 
$\Fmf(\ELH,\fixedAbox^n,\AQ)$  shatters $\{(\fixedAbox^n,A(a_i))\mid 1\leq i \leq n\}$. 
 This does
 not
hold if we add $s(a_n,a_n)$
 to
 $\fixedAbox^n$.
 \hfill {\mbox{$\triangleleft$}}
\end{example}

\begin{figure}
	\centering
	\begin{tikzpicture}
	\tikzset{every node/.style={circle,fill=black,inner sep=0pt,minimum size=4pt}}
	\tikzset{every edge/.append style={every node/.style={circle,fill=white,inner sep=0pt,minimum size=0pt}, >=stealth}}

	\node (1) at (0,0) {} ;
	\node[fill=white,opacity=0,text opacity=1] () at (0,0.2) { $ a_1  $} ;
	\node (2) at (2,0) {} ;
	\node[fill=white,opacity=0,text opacity=1] () at (2,0.2) { $ a_2  $} ;
	\draw[->, >=stealth] (1) to (2) ;
	\node[fill=white,opacity=0,text opacity=1] () at (1,0.15) { $ r  $} ;
	
	\draw[->, >=stealth]  (1) edge [loop left] node[fill=white,opacity=0,text opacity=1] {$ s $} () ;
	
	\draw[->, >=stealth] (2) edge[bend left=25] node[label={[xshift=.0mm, yshift=.0mm]$ r $}]{} (1);
	\end{tikzpicture}
	\caption{For $ \fixedAbox^2 $, 
	$ \setexamples = \{(\fixedAbox,A(a_1)), (\fixedAbox,A(a_2)) \} $  is shattered because 
		we can find in \hypothesisSpace:  
		$ h_1 =\{\exists s.\top \sqcap \exists r.\exists s.\top\sqsubseteq A \} $, 
		$ h_2 =\{\exists r.\exists s.\top \sqsubseteq A \} $, 
		$ h_3 = \{\exists s.\top\sqsubseteq A \} $, $ h_4 = h_2 \cup h_3$. }
	\label{fig:vcdimension}
\end{figure}

For discrete cases, in particular, for fragments of first-order logic, 
the lower bounds obtained with the VC-dimension cannot be larger than 
the size of the learned expressions assuming a reasonable encoding scheme~\cite{DBLP:journals/ml/AriasK06}. 
The authors argue that many VC-dimension bounds in the literature showing exponential or infinite growth 
are in terms 
of some parameters (number of clauses, etc.) determining the size of the target, 
while other parameters (number of literals, etc.)
are ignored. 
Let $\Fmf^{m}=(\examples, \set, \hypothesisSpace, \mu)$ be a learning framework 
where the string size of the elements of $\hypothesisSpace$ is bounded by $m$.  
Since the VC-dimension is bounded by a logarithm of $|\hypothesisSpace|$ (for $\hypothesisSpace$ finite),
 $\vc{\Fmf^{m}}=O(m)$~\cite{Vapnik:1995:NSL:211359}.

\begin{proposition}
For all $m\in\Nbb$, $\Fmf^{m}(\ELH,\fixedAbox,\CQ)$ is  
 PAC learnable with a polynomial number of example queries. 
\end{proposition}
Since the sample complexity (number of classified examples) is polynomial in the size of 
the target, polynomial time PAC learnability amounts to showing that one can compute a hypothesis in $\hypothesisSpace$ 
that is consistent with the classification of the examples
in polynomial time.  
However, even if \fixedAbox is fixed, checking whether 
$(\fixedAbox,q)$ is a positive example for a hypothesis \Hmc
is \NP-hard if the underlying structure of $\fixedAbox$ is non-bipartite~\cite{Hell:1990:CHC:80298.80312}.
So (unless $\textsc{P}\xspace=\NP$) there is not much hope for polynomial time
learnability, even with membership queries, since in this case 
one may not be able to convert the
CQ
into an 
IQ
(as we did in Theorem~\ref{thm:cqrupper}).

\section{Conclusion}\label{sec:discussion}   
 
 We introduced the OMQA learning setting and investigated 
 the complexity of learning \ELH ontologies in this setting 
 with  various 
 query languages. 
We then considered what happens when the data changes 
and adaptations to 
settings where 
the algorithm learns from classified data, limiting interactions with oracles. 
Our positive result for \IQ-inseparable \ELH TBoxes paves the way for further studies 
on the complexity of learning ontologies formulated in   more expressive languages. 
We leave the problem of exactly learning \ELH TBoxes 
with \CQrs as an open problem. 
Learning with a more expressive query 
language is \emph{not easier} because the oracle can formulate counterexamples 
which are not informative. Neither it is more difficult because on the other hand, with 
a more expressive language,  
the learner can 
pose more informative membership queries.
It would
also
be interesting to 
investigate a similar data model in which the ABox
is fixed
for all the examples,
so that 
the data pool contains examples in the form of queries alone.

\section*{Acknowledgments}
This research has been supported by the Free University of Bozen-Bolzano through 
the projects PACO and MLEARN.

\bibliographystyle{aaai}
\fontsize{9.6pt}{10.6pt}\selectfont
\bibliography{references_shrink}

\iftechrep
\clearpage
 \appendix
 \section{Proofs for Section ``Polynomial Learnability''} 
We introduce some basic definitions and lemmas that will be used in our proofs. 
We are going to use  the definition of the 
tree interpretation of a concept, the notion of a canonical model, of a homomorphism, 
and of a simulation.

\begin{definition}[Canonical model of an ABox] The canonical model $\I_\A=(\Delta^{\I_\A},\cdot^{\I_\A})$ of an ABox \A is defined as follows:
	\begin{itemize}
		\item $\Delta^{\I_\A} =  \individuals{\A}$;
		\item $a^\I = a$ for all individuals $a \in \individuals{\A} $;
		\item $A^{\I_\A} = \{a \mid A(a) \in \A\}$;
		\item $r^{\I_\A} = \{(a,b) \mid r(a,b) \in \A \}$.
	\end{itemize}
\end{definition}

A path in a \ELH concept expression $C$ is a finite sequence of the form 
$C_0 \cdot r_0 \cdot C_1 \cdot r_1 \cdots r_n \cdot C_n $ 
where $C_i$ is a concept expression, $r_i$ is a role name, $C_0 = C$, and, for  all $i \in \{0,\ldots, n-1\}$, the 
concept expression $\exists r_{i+1}.C_{i+1}$ is a top-level conjunct of $C_i$. 
The set $\paths(C)$ contains all paths in $C$.
The set $\tail(p) = \{A\mid A \in\NC$ \text{ is a top-level conjunct of } $C_k$ \text{ where } $C_k$ \text{ is the last concept expression in } $p\}$.

\begin{definition}[Tree interpretation] Let $C$ be an \ELH concept expression.
	The tree interpretation $\I_C=(\Delta^{\I_C},\cdot^{\I_C})$ of a concept $C$ is defined as follows:
	\begin{itemize}
		\item $\Delta^{\I_C} = \paths(C)$;
		\item $A^{\I_C} = \{p \in \paths(C) \mid A \in \tail(C)\}$; 
		\item $r^{\I_C} = \{p \in \paths(C) \times \paths(C)  \mid \text{there exists }p' = p \cdot r \cdot D \text{ and } p,p'\in\Delta^{\I_C} \text{ for some concept expression }  D \} $.
	\end{itemize}
	We may denote the root path $C$ of $\I_C$ with $\rho_C$.
\end{definition}

\begin{definition}[Canonical model of a KB]
	\label{d:canonicalModel}
	The canonical model $\I_{\T,\A}$ of an \ELH KB $(\T,\A)$ is defined inductively. 
	For each $ r\in \NR $, $\I_0$ is defined by extending the canonical model $\I_\A$ of an ABox \A with
	\[r^{\I_0} = \{(a,b) \mid s(a,b) \in \A\text{ and } \T\models s\sqsubseteq r \}.\]
	Assume that $ \I_i $ has been defined. We define $ \I_{i+1} $ as follows.	
	If there exist $C \sqsubseteq D \in \T$ such that $p
	\in \Delta^{\I_i}, p \in C^{\I_i}, p \not\in D^{\I_i}$ and $D = 
	\bigsqcap_{1 \leq j \leq l} A_j \sqcap  \bigsqcap_{1 \leq j' \leq l'} \exists s_{j'}.D_{j'}$, we create $\I_{i+1}$ in the following way:
	\begin{itemize}
		\item $\Delta^{\I_{i+1}} = \Delta^{\I_i} \cup \{p\cdot s_{j'} \cdot q \mid q \in \paths(D_{j'}), 1 \leq j' \leq l' \}$;
		\item $A^{\I_{i+1}} = A^{\I_{i}} \cup \\
		\{p\cdot s_{j'} \cdot q \mid q \in \paths(D_{j'}),\; A \in \tail(q),\; 1 \leq j' \leq l' \} 
		\cup \\
		\{ p \mid A_j = A, 1 \leq j \leq l\} $;
		\item $r^{\I_{i+1}} = r^{\I_{i}} \cup \\
		\{(p\cdot s_{j'} \cdot q),(p\cdot s_{j'} \cdot q') \mid (q,q') \in s^{\I_{D_{j'}}},  \T \models s\sqsubseteq r , 1 \leq j' \leq l' \}  
		\cup \\
		\{ (p,p\cdot s_{j'} \cdot D_{j'}) \mid \T \models s_{j'} \sqsubseteq r, 1 \leq j' \leq l'\}$.
	\end{itemize}
	Finally, the canonical model of a KB $(\T,\A)$ is $\I_{\T,\A} = \bigcup_{i=0}^\infty \I_i$.
\end{definition}

\begin{definition}[Homomorphism] Let $\Imc$ be an interpretation and let \treedef{C} be the tree representation of concept expression $ C $. 
	A homomorphism $h: \tree{C} \rightarrow  \Imc$ is a   function from $\treeV{C}$ to $\Delta^\Imc$ such that 
	\begin{itemize}
		\item for all $\nu \in \treeV{C}$ and $A \in \NC$, if $A \in \treel{C}(\nu)$ then $h(\nu) \in A^\Imc$;
		\item for all $\nu_1,\nu_2 \in\treeV{C}$ and $r \in \NR$, if $\treel{C}(\nu_1,\nu_2) =r$ then $(h(\nu_1),h(\nu_2)) \in r^\Imc$.
	\end{itemize}
\end{definition}

\begin{definition}[Simulation] Let $\I=(\Delta^\I,\cdot^\I),\Jmc=(\Delta^\Jmc,\cdot^\Jmc)$ be two interpretations. 
A simulation S is a non-empty relation $S\subseteq \Delta^\I \times \Delta^\Jmc$ satisfying the following properties: 
	\begin{itemize}
		\item for all concept names $A \in \NC$ and all $(d,e) \in S $ it holds that if $\\ d \in A^\I$ then  $e \in A^\Jmc$.
		\item for all role names $r \in \NR$, if $(d,e) \in S$ and $(d,d_1) \in r^\I, d_1 \in \Delta^\I$ then there exists $e_1 \in \Delta^\Jmc$ such that $(e,e_1) \in r^\Jmc$ and $(d_1,e_1) \in S$.
	\end{itemize}
	If \Imc and \Jmc are tree interpretations, we write $\Imc\Rightarrow\Jmc$ 
	if there is a simulation $S\subseteq \Delta^\I \times \Delta^\Jmc$ containing 
	the pair with the roots of \Imc and \Jmc.
\end{definition}

The following  standard lemmas 
are going to be used throughout the remaining   of this section in order to show properties of the algorithms.

\begin{lemma}
	\label{lemmabisimulation}
	If $ (\I_1,d_1) \sim (\I_2,d_2)$, then the following holds for all \ELH concept expressions $ C $: 
	$ d_1\in C^{\I_1} $ iff $ d_2 \in C^{\I_2} $.
	Moreover, given an arbitrary TBox
	\T and  ABoxes $ \A_1 $ and $ \A_2 $,  
 $ (\I_{\A_1},d_1) \sim (\I_{\A_2},d_2)$ implies that, for all \ELH concept expressions $ C $: 
 $ d_1\in C^{\I_{\T,\A_1}} $ iff $ d_2 \in C^{\I_{\T,\A_2}} $.
\end{lemma}

\begin{lemma}
	\label{l:homomorhelp}
	Let $ C $ be an \ELH concept expression and let $ \I $ be an interpretation with $ d\in\Delta^\I $. Then, 
	$ d\in C^\I $ if, and only if, there is a homomorphism $ h: \tree{C} \rightarrow \I $ such that $ h(\rho_C)=d $.
\end{lemma}

\begin{lemma}
	\label{l:subsumptionHomomorphism}
	Let $C$ be an \ELH concept expression, \T a TBox and \A an ABox. The following statements are equivalent:
	\begin{enumerate}
		\item $\T\models C \sqsubseteq D$;
		\item $\rho_C \in D^{\I_{\T,C}}$;
		\item There is a homomorphism $h: \tree{D} \rightarrow \I_{\T,C} $ such that $h(\rho_D)=\rho_C$.
	\end{enumerate}
	
\end{lemma}

\subsection{Learning \ELH ontologies with \AQ}
We now present in full detail a learning algorithm for 
$\Fmf(\ELH,\fixedAbox,\AQ)$. 
Our algorithm is based on the approach used to learn a fragment of 
\ELH where complex concept expressions are only allowed on the left side of inclusions~\cite{DBLP:conf/aaai/KonevOW16}. 
Algorithm \ref{a:learningAQ} shows the steps that the learner should do to learn an \AQ-inseparable TBox. 
It first computes an initial part 
of the hypothesis $ \hypothesisOntology $ which consists of all RIs 
and all CIs with concept names  on both sides entailed by the target ontology \targetOntology. 
This initial $ \hypothesisOntology $ is constructed by asking $O(|\Sigma_\Tmc|^2)$  membership queries.
More precisely, in this phase the learner  
calls $ {\sf MQ}_{\Fmf(\ELH,\fixedAbox,\AQ),\targetOntology}$ with  $(\{A(a)\},B(a)) $ as input, for all $ A,B\in\Sigma_\Tmc\cap\NC $. If the answer is positive, 
the learner adds $ A\sqsubseteq B $ to $ \hypothesisOntology $. In this way the learner adds 
to \hypothesisOntology all CIs of the form $ A\sqsubseteq B $ entailed by \targetOntology. After that, 
similarly, for each combination  $ r,s\in\Sigma_\Tmc\cap\NR $, the learner calls $ {\sf MQ}_{\Fmf(\ELH,\fixedAbox,\AQ),\targetOntology}$ 
with  $(\{r(a,b)\},s(a,b)) $ as input and, if the answer is positive, it adds $ r\sqsubseteq s $ to  \hypothesisOntology. 
As a consequence, the learner learns all RIs and the next counterexamples, if any, will only be of the form $ (\fixedAbox,A(a)) $.

\medskip

\begin{algorithm}
	\label{a:learningAQ}
	\SetAlgoLined
	\LinesNumbered
	\DontPrintSemicolon
	\KwIn{Signature $ \Sigma_\Tmc $}
	\KwOut{Hypothesis \hypothesisOntology}
	$\hypothesisOntology = \{A \sqsubseteq B \mid \T\models A \sqsubseteq B, \text{   } A,B\in\Sigma_\targetOntology\}
	  \cup \{r \sqsubseteq s \mid \T\models r \sqsubseteq s, \text{   } r,s\in\Sigma_\targetOntology \}$\;\label{a:learning:basic}
	\While{$ (\T,\fixedAbox) \not\equiv_{\AQ} (\hypothesisOntology,\fixedAbox) $}{\label{a:learningAQ:loop}
		Let $(\fixedAbox ,A(a))$ be a positive counterexample\; \label{ln:pce}
		$\A\leftarrow \text{TreeShape}(\Sigma_\Tmc, \A,\hypothesisOntology)$\; \label{a:learningAQ:unfold}
		Find $B$ such that $\T\models C_{\A} \sqsubseteq B$ and $\hypothesisOntology\not\models C_{\A} \sqsubseteq B$, 
		 add $C_{\A} \sqsubseteq B$ to \hypothesisOntology\;\label{a:learningAQ:add}
		
	}
	\Return \hypothesisOntology\;	
	\caption{\ELH \AQ learning algorithm}
\end{algorithm}

\medskip

After this initial step, 
the learner enters in a while loop and asks 
inseparability queries in order to check if it has already found an \AQ-inseparable hypothesis. If this is true, 
it stops and returns an AQ-inseparable TBox, otherwise it gets  a positive counterexample $ (\fixedAbox,A(a)) $ from the oracle, 
which is   used to add   missing knowledge to the hypothesis $ \hypothesisOntology $. 
The assumption in Line~\ref{ln:pce} that  counterexamples are \emph{positive}   
 is justified by the fact that the algorithm maintains the invariant that 
$\targetOntology\models\hypothesisOntology$. 
Indeed, the algorithm adds only CIs that are logical consequences of the target \T . 
This property holds before the first inseparability query is asked and 
it remains true when the next (positive) counterexamples CIs are received. 
When a new counterexample $ (\fixedAbox ,A(a)) $ 
is received, the learner should find a tree shaped ABox 
\A rooted in $\rho$ obtained   from \fixedAbox such that there is 
$B\in\Sigma_\T\cap\NC$ with  $ (\T,\A)\models B(\rho) $ and $ (\hypothesisOntology,\A)\not\models B(\rho) $, 
in order to 
 add an informative CI $C_\A\sqsubseteq B$ to \hypothesisOntology. We would like $ |C_\A|$ to be polynomial in $|\T| $ 
so that  $|\hypothesisOntology|$ is polynomial in $|\T|$.

	\SetAlFnt{\normalsize}
	\vspace{0pt}
	\hspace{-15pt} 
	\begin{algorithm}
		\label{a:treeShape}
		\SetAlgoLined
		\LinesNumbered
		\DontPrintSemicolon
		\KwIn{Signature $ \Sigma_\Tmc $, ABox $\A$,  TBox \hypothesisOntology}
		\KwOut{ABox}
		\vspace{3pt}
		$ \A \leftarrow$Minimize$(\Sigma_\Tmc,\A,\hypothesisOntology)$\;
		\vspace{2pt}
		\While{there is a cycle $ c\in $ \A}{
		\vspace{4pt}
			$ \A \leftarrow$Unfold($ \Sigma_\Tmc, c,\A$)\;\label{a:treeShape:unfold}
			\vspace{3pt}
			$ \A \leftarrow$Minimize($\Sigma_\Tmc,\A,\hypothesisOntology$)\; \label{a:treeShape:minimize}
		}
		\vspace{3pt}
		\Return $\A$\;
		\caption{TreeShape}
	\end{algorithm}
	\SetAlFnt{\scriptsize}
	\vspace{0pt}
\hspace{-5pt} 
\SetAlFnt{\normalsize} 
\begin{algorithm}
	\label{a:minimize}
	
	\SetAlgoLined
	\LinesNumbered
	\DontPrintSemicolon
	\KwIn{Signature $ \Sigma_\Tmc $, ABox $\A$,  TBox \hypothesisOntology}
	\KwOut{ABox}
	 Saturate \A with \hypothesisOntology\;\label{a:minimize:saturate}
	\ForEach{$ A\in\Sigma_\Tmc \cap \NC $ and $ a\in \individuals{\A} $ such that $ (\T,\A) \models A(a)$ and $ (\hypothesisOntology,\A) \not\models A(a)$}{\label{a:minimize:check}
		Domain Minimize \A with $A(a)$\;\label{a:minimize:dommin}
		Role Minimize \A with $A(a)$\;\label{a:minimize:rolmin}
	}
	\Return $\A$\;
	\caption{Minimize}
\end{algorithm}

\medskip

In Line \ref{a:learningAQ:unfold} of Algorithm \ref{a:learningAQ}, Algorithm \ref{a:treeShape} is 
called in order to find a tree shaped ABox. ``Unfold'' (defined next), doubles cycles in \A 
and ``Minimize'' removes redundant assertions~\cite{DBLP:conf/aaai/KonevOW16}.

\mypar{Unfold} We say that \Amc has a (undirected) cycle if there is a finite sequence 
$a_0 \cdot r_1 \cdot a_1 \cdot ... \cdot r_k \cdot a_k$ such that (i) 
$a_0 = a_k$ and (ii)  
there are mutually distinct assertions of the form 
$ r_{i+1}(a_i,a_{i+1})$ or $ r_{i+1}(a_{i+1},a_i)$ in \Amc, for $0 \leq i < k$. 
For a cycle $c = a_0 \cdot r_1 \cdot a_1 \cdot ... \cdot r_k \cdot a_k$, 
denote as ${\sf nodes}(c) = \{a_0,a_1,...,a_{k-1}\}$ the set of individuals that 
occur in $c$. 
We denote by $\widehat{a}$ the copy of an element $a$ created by the 
unfolding cycle operation described below. The set of copies of 
individuals that occur in $c$ is denoted by 
${\sf nodes}(\widehat{c}) = \{\widehat{a_0},\widehat{a_1},...,\widehat{a}_{k-1}\} $.  
Let $\Imc_{\Amc}$ be the canonical interpretation of an ABox \Amc. 
An element $a \in \Delta^{\Imc_{\Amc}}$ is \emph{folded} if 
there is a cycle $c = a_0 \cdot r_1 \cdot a_1 \cdot ... \cdot r_k \cdot a_k$ with 
$a =a_0=a_k$. Without loss of generality we assume that 
$r_1(a_0,a_1) \in \Amc$. 
The \textit{cycle unfolding} of  $c$ 
is described below.

\begin{enumerate}
\item We first open the cycle by removing $r_1(a_0,a_1)$ from 
$\Amc$. So $r_1^{\Imc_{\Amc}}:=r_1^{\Imc_{\Amc}} \setminus \{(a_0,a_1)\}$.  
\item Then we create copies of the nodes in the cycle:
\begin{itemize}
\item $\Delta^{\Imc_{\Amc}} := \Delta^{\Imc_{\Amc}} \cup \{\widehat{b} \mid b \in {\sf nodes}(c) \}$
\item $A^{\Imc_{\Amc}} := A^{\Imc_{\Amc}} \cup \{\widehat{b} \mid b \in A^{\Imc_{\Amc}}\}$
\item $r^{\Imc_{\Amc}} := r^{\Imc_{\Amc}} \cup \{(\widehat{b},\widehat{d})\mid (b,d) \in r^{\Imc_{\Amc}}\} $ 
$\cup \{(\widehat{b},e)\mid (b,e) \in r^{\Imc_{\Amc}}$, $e \notin {\sf nodes}(c)\} $
\end{itemize}
\item As a third step we close again the cycle, now with double size.
So we update $r_1^{\Imc_{\Amc}}:=r_1^{\Imc_{\Amc}} \cup \{(a_0,\widehat{a_1}),(\widehat{a_0},a_1)\}$.  
\end{enumerate}

\mypar{Minimize}	Given an ABox \A, we denote by 
	$ \A^{-\alpha} $ the result of removing assertion $ \alpha $ from \A.
	Also, 
	given $ a\in \individuals{\A} $, we denote by $ \A^{-a}$ the result of removing from \A all ABox assertions in which $ a $ occurs.
The ABox transformation rules (exhaustively applied) in Algorithm \ref{a:minimize} are defined as follows:
\begin{itemize}
	\item (Saturate \A with \hypothesisOntology)
	Let $ \A $ be an ABox, let  \hypothesisOntology be a TBox, and let $\alpha$ be an assertion built 
	from symbols in $\Sigma_\T$ and $\individuals{\A} $. If 
	$ \alpha \not\in \A$ and $ (\hypothesisOntology,\A )\models\alpha$ 
	then replace $\A $ with $ \A\cup \{\alpha\} $; 
	\item (Domain minimize \A with $ A(a) $)
	If $ (\A,A(a)) $ is a positive example and $ (\T,\A^{-b})\models A(a) $ then replace $ \A $ by $ \A^{-b} $ 
	where $b\in\individuals{\A}$;
	\item (Role minimize \A with $ A(a) $) If $ (\A,A(a)) $ is a positive example and $ (\T,\A^{-r(b_1,b_2)})\models A(a) $ then replace $ \A $ by $ \A^{-r(b_1,b_2)} $.
\end{itemize}
We say that an ABox is \emph{minimal} if there is $A\in\Sigma_\T\cap\NC$ and $a\in\individuals{\A}$ such that 
$(\T,\A)\models A(a)$, $(\Hmc,\A)\not\models B(b)$ and the rules `Domain minimize \A with $ A(a) $'
and `Role minimize \A with $ B(b) $' have been exhaustively applied 
with all $B\in\Sigma_\T\cap\NC$ and all $b\in\individuals{\A}$. 

\begin{lemma}\cite[Lemma~49]{DBLP:conf/aaai/KonevOW16}\label{lem:unfold-2}
Let $\Amc'$ be the result of unfolding a cycle $c$ in \Amc. 
Then the following relation $S \subseteq \Delta^{\Imc_{\Amc}} \times \Delta^{\Imc_{\Amc'}}$ 
is a simulation $\Imc_\Amc \Rightarrow \Imc_{\Amc'}$:
\begin{itemize}
\item for $a \in \Delta^{\Imc_{\Amc}} \setminus {\sf nodes}(c)$, 
$(a^{\Imc_{\Amc}},a^{\Imc_{\Amc'}}) \in S$;
\item for $a \in {\sf nodes}(c)$, 
$(a^{\Imc_{\Amc}},a^{\Imc_{\Amc'}}) \in S$ and 
$(a^{\Imc_{\Amc}},\widehat{a}^{\Imc_{\Amc'}}) \in S$.
\end{itemize} 
\end{lemma}

\begin{lemma}\cite[Lemma~50]{DBLP:conf/aaai/KonevOW16}\label{lem:unfold-1}
Let $\Amc'$ be the result of unfolding a cycle $c$ in \Amc. 
Let $h_\ast :\Imc_{\Amc'} \rightarrow \Imc_{\Amc}$ be 
the following mapping:
\begin{itemize}
\item  for $a \in \Delta^{\Imc_{\Amc}} \setminus {\sf nodes}(c)$, 
$h_\ast(a^{\Imc_{\Amc'}})=a^{\Imc_{\Amc}}$;
\item for $a \in {\sf nodes}(c)$, $h_\ast(a^{\Imc_{\Amc'}})=a^{\Imc_{\Amc}}$ and  $h_\ast(\widehat{a}^{\Imc_{\Amc'}})=a^{\Imc_{\Amc}}$.
\end{itemize} 
Then, $h_\ast :\Imc_{\Amc'} \rightarrow \Imc_{\Amc}$ is a homomorphism.
\end{lemma}

\begin{remark}\label{rem}\upshape
In the domain and role minimization steps, we assume that the algorithm always 
attempts to remove the `clones' (that is, elements in ${\sf nodes}(\widehat{c})$)
 generated in the unfold step first (that is, before attempting to remove elements of $\individuals{\fixedAbox}$). 
By Lemmas~\ref{lem:unfold-2} and~\ref{lem:unfold-1} this assumption is w.l.o.g.
This means that we can assume  that there is a positive example $(\A,A(a))$ with 
$a\in\individuals{\fixedAbox}$
whenever an ABox \A is
returned by Algorithm~\ref{a:treeShape}.
\end{remark}

\begin{lemma}
	\label{l:outputMinimizePositiveFixed}
	If $ (\T,\A)\models A(a) $ and  $ (\hypothesisOntology,\A)\not\models A(a) $
	where \A is the output of Algorithm \ref{a:minimize}, 
	then $ (\fixedAbox,A(a)) $ is a positive counterexample.
\end{lemma}
\begin{proof}
Assume $ (\T,\A)\models A(a) $ and  $ (\hypothesisOntology,\A)\not\models A(a) $
	where \A is the output of Algorithm \ref{a:minimize}.
By Lemma~\ref{lem:unfold-1} 
and the fact that minimization only removes nodes and role assertions from \A, there is an ABox homomorphism from $\A$ to ${\fixedAbox}$, 
and so, $ (\T,\fixedAbox)\models A(a) $ (by Remark~\ref{rem}, 
we can assume that $a\in\individuals{\fixedAbox}$).
We need to show that $ (\hypothesisOntology,\fixedAbox)\not\models A(a) $.
In Line~\ref{a:minimize:saturate} the algorithm  saturates \fixedAbox with \hypothesisOntology. 
This means that if  $ (\hypothesisOntology,\fixedAbox)\models A(a) $ then 
$A(a)\in \fixedAbox$ (after  saturating \fixedAbox with \hypothesisOntology). 
By construction of \A, if $A(a)\in \fixedAbox$ then $A(a)\in \A$, which contradicts 
the fact that $ (\hypothesisOntology,\A)\not\models A(a) $. So 
$ (\hypothesisOntology,\fixedAbox)\not\models A(a) $. 
\end{proof}

By Lemma~\ref{l:outputMinimizePositiveFixed} and the fact that $\A$ is tree shaped, we can see that 
one can compute 
a CI $C_\A\sqsubseteq B$ such that 
  $(\targetOntology,\fixedAbox)\models B(b)$, $(\hypothesisOntology,\fixedAbox)\not\models B(b)$, and
$(\hypothesisOntology\cup\{C_\A\sqsubseteq B\},\fixedAbox)\models B(b)$, for some 
$b\in\individuals{\fixedAbox}$. 
We now show that 
one can compute such CI in polynomial time in $|\fixedAbox|$ and $|\targetOntology|$.

\begin{lemma}
	\label{l:polmanyqueries}
	For any \ELH target \T and any \ELH hypothesis \hypothesisOntology  
	with size polynomial in $|\T|$, given a positive counterexample $ (\fixedAbox,A(a)) $, 
	Algorithm \ref{a:minimize} terminates in polynomial time 
	in $ |\fixedAbox| $ and $ |\T| $. 
\end{lemma}
\begin{proof}
	Entailment in \ELH is in \PTime~\cite{dlhandbook} and  saturating an 
	ABox with \hypothesisOntology does not require any query to the oracle, 
	thus all steps in  saturation with \hypothesisOntology can  be 
	computed in polynomial time. 
	For each concept name   $A\in \Sigma_\Tmc$, domain minimization 
	asks at most $ |\individuals{\fixedAbox}| $ membership queries 
	and role minimization asks at most $ |\fixedAbox| $ membership 
	queries (each query to the oracle 
counts as one computation step). 
\end{proof}

We now show that Algorithm~\ref{a:treeShape} also runs in polynomial time in $ |\fixedAbox| $ and $ |\T| $. 
In each iteration, unfold adds new individuals to the ABox and minimize removes 
individuals. To show that the algorithm terminates (in polynomial time), 
we show that on one hand the size of an ABox after minimization is polynomial in  $|\Tmc|$.
On the other hand, we show that each time the algorithm unfolds the number of individuals 
after minimization is larger than before the unfolding step (see~\cite{DBLP:conf/aaai/KonevOW16}).

\begin{lemma}
	\label{l:conceptinT}
	Let \T be an \ELH  ontology, $ C' $ a concept expression and 
	$ A $ a concept name such that $ \emptyset\not\models C'\sqsubseteq A $. If 
	$ \T\models C'\sqsubseteq A $ then there is  $ C\sqsubseteq A\in \T$ such that $ \T\models C'\sqsubseteq C $.
\end{lemma}
\begin{proof}
	Let $ \I_{\T,C'} $ be the canonical model of $ C' $ with respect to \T. We have that 
	 $ \rho_{C'}\in A^{{\I_{\T,C'}}} $ because $ \T\models C'\sqsubseteq A $. Since  $ \emptyset\not\models C'\sqsubseteq A $, 
	  by the construction of the canonical model, there is  $ C\sqsubseteq A\in\T $ 
	  such that 
	  there is a homomorphism $ h:\tree{C}\rightarrow \I_{\T,C'}  $ 
	  with  $ h(\rho_{C})=\rho_{C'} $. Then, by \refLemma{\ref{l:subsumptionHomomorphism}}, $ \T\models C'\sqsubseteq C $.
\end{proof}

Given a concept expression $C$, we write $C'\prec C$ if $C'$ is the concept expression 
corresponding to the tree that results from replacing 
a subtree $T_D$ of $T_C$ by a concept name $A$ such that $\T\models A\sqsubseteq D$.
We write $C'\prec^{\ast} C$ if $C'=C$ or there is a sequence $C_1\prec C_2\prec \ldots\prec C_n$
with $C_1=C'$, $C_n=C$ and $n>1$. 
By this definition, the following lemma is straightforward. 
\begin{lemma}\label{lem:min}
For all concept expressions  $C',C$, we have that $C'\prec^{\ast} C$ iff  
$\T\models C'\sqsubseteq C$. Moreover, $|C'|\leq |C|$. 
\end{lemma}

\begin{lemma}
	\label{l:maboxpol}
	Let \A be a minimal ABox. 
	Then $ |\individuals{\A}| \leq |\T| $.
\end{lemma}
\begin{proof}
	Let $ \A $ be the ABox returned by Algorithm \ref{a:minimize}.
	Then there is $A(a)$ such that $(\T,\A)\models A(a)$ and 
		$(\hypothesisOntology,\A)\not\models A(a)$. 
 This means that there is a CI $ C'\sqsubseteq A $ such that
 $\T\models  C'\sqsubseteq A$,  $\Hmc\not\models C'\sqsubseteq A$
 and $a \in (C' \setminus A)^{\Imc_{\Amc}}$.
 By Lemma~\ref{l:conceptinT}, there is
 $C \sqsubseteq A \in \Tmc$ such that $\T\models C'\sqsubseteq C$.
By Lemma~\ref{lem:min},  
 $C'\prec^\ast C$ and $|C'|\leq |C|$.
If $a \in C'^{\Imc_{\Amc}}$ then 
 (by Lemma \ref{l:subsumptionHomomorphism}) there is a homomorphism $h : \Imc_{C'} \rightarrow \Imc_{\Amc}$ 
mapping $\rho_C$ to $a$, where $\rho_C$ is the root of $\Imc_C$. 
Since $|C'|\leq |C|$, we only need to show 
that $h$ is surjective.
Suppose this is not the case. 
Then, there is $b \in \Delta^{\Imc_{\Amc}}$ such that 
$b \notin {\sf Im_h}$, where ${\sf Im_h} = \{e \in \Delta^{\Imc_{\Amc}} \mid 
e = h(p)$ for some $ p \in \Delta^{\Imc_{C'}}\}$. 

Denote as $\Imc_{\Amc^{-b}}$ the result of removing 
$b \notin {\sf Im_h}$ from $\Imc_{\Amc}$. Since  
 $\Imc_{\Amc^{-b}}$ is a subinterpretation of $\Imc_{\Amc}$, if $a \notin A^{\Imc_{\Amc}}$ then 
$a \notin A^{\Imc_{\Amc^{-b}}}$. So  
$a \in (C' \setminus A)^{\Imc_{\Amc^{-b}}}$, which means that 
$(\T,\A^{-b})\models A(a)$ while we still have
		$(\hypothesisOntology,\A^{-b})\not\models A(a)$. 
		This contradicts the fact that \A is minimal. 
Thus, $|\individuals{\Amc}| \leq |\Delta^{\Imc_{C'}}|\leq |\Delta^{\Imc_{C}}| \leq |\Tmc|$. 
\end{proof}

\begin{lemma}\label{lem:iteration}
Let $\A_n$ be the    minimal ABox 
computed in the $n$-th iteration 
in Line~\ref{a:treeShape:minimize} of Algorithm~\ref{a:treeShape}. 
Assume $\A_n$ has a cycle.  
For all $n \geq 0$, $|\individuals{\A_{n+1}}|>|\individuals{\A_{n}}|$.
\end{lemma}
\begin{proof}
The argument is similar to the one presented in~\cite[Lemma~51]{DBLP:conf/aaai/KonevOW16}. 
There is an ABox homomorphism from $\A_{n+1}$  to $\A_n$, which is surjective 
because otherwise  $\A_n$ would not be minimal. This homomorphism is not injective because 
at least one element of a cycle in $\A_n$ has been unfolded and remained in $\A_{n+1}$ after minimization.  
\end{proof}

Lemmas~\ref{l:maboxpol} and \ref{lem:iteration}  bound the number of iterations of Algorithm~\ref{a:treeShape}.
We already argued in Lemma~\ref{l:polmanyqueries} that Algorithm~\ref{a:minimize} 
terminates in polynomial time (in $|\T|$ and $|\fixedAbox|$) and it is easy to see that unfolding can also 
be performed in polynomial time (first it doubles a cycle in \fixedAbox and in subsequent iterations 
it doubles a cycle in an ABox of size polynomial in $|\T|$). 
Thus, Algorithm~\ref{a:treeShape} also runs in polynomial time.

\subsection{Learning \ELH ontologies with \IQ}
Our proof strategy is based on the proof for the learning framework $\Fmf(\ELH_{\sf rhs}, \IQ)$~\cite{DBLP:conf/aaai/KonevOW16}.
 We assume w.l.o.g. that the target \T does not entail non-trivial role equivalences (that is, 
 there are no distinct $r$ and $s$ such that $\T \models r\sqsubseteq s$ and $\T \models s\sqsubseteq r$).
 This simplifies our presentation, in particular, Line~\ref{a:learningIQ:t-essential} of Algorithm~\ref{a:learningIQ} 
 relies on it. The assumption can be dropped using the notion of a
 representative role (see~\cite[Theorem~34]{KLOW18}). 
 
The learning algorithm for $\Fmf(\ELH,\fixedAbox,\IQ)$ is given by Algorithm~\ref{a:learningIQ}.
In Line~\ref{a:learningIQ:AQ}, Algorithm~\ref{a:learningIQ} computes a 
hypothesis \hypothesisOntology that is \AQ-inseparable from \T (with only membership queries).
Then, the algorithm iterates in a `while loop' posing  inseparability queries regarding the \IQ language. 
Upon receiving a positive counterexample $ (\fixedAbox,C(a)) $, 
some operations are made in order to learn a CI of the form $A\sqsubseteq D$ 
entailed by \targetOntology but not by \hypothesisOntology. 
The assumption in Line~\ref{a:learningIQ:ex} that  counterexamples are \emph{positive}   
 is justified by the fact that the algorithm maintains the invariant that 
$\targetOntology\models\hypothesisOntology$. Indeed, in the OMQA setting, 
a learning algorithm can always ensure this property by 
keeping in the hypothesis only CIs $C\sqsubseteq D$ such that 
a membership query with $(\A_C,D(\rho_C))$ as input has returned `yes'.
In particular, Algorithm~\ref{a:learningIQ} only adds to the hypothesis 
CIs that are entailed by \T.

We now explain the notion of a \Tmc-essential CI~\cite{KLOW18} which appears in 
Algorithm~\ref{a:learningIQ}. 
This notion is based on some operations performed on examples 
of the form $(\{A(a)\},C(a))$
representing CIs 
 $A\sqsubseteq C$, 
where $A\in \NC$, $a\in\NI$, and $C$ an arbitrary 
concept expression. Intuitively, a \Tmc-essential  CI is a CI that is 
in a sense informative for the learning algorithm
and its size is bounded by a polynomial in \T. 
For the ontology language \ELH, 
the operations are:
\emph{concept saturation for \T}, \emph{role saturation for \T}, \emph{sibling merging for \T} and 
\emph{decomposition on the right for \T}. We may omit `for \T' if this is clear from the context. 
We now recall  these operations~\cite{KLOW18}, 
adapted to the OMQA setting. In the following, assume we are given 
an example of the form $(\{A(a)\}, C(a))$.

Concept saturation for \T consists of updating $C$ with the result $C'$ of  
   choosing a node $ \nu $ in $ T_C $ and adding a new concept name from $\Sigma_\T$ to the label 
 of $ \nu $   if $ (\{A(a)\},C'(a)) $ 
 is still a positive example. 
 A positive example $ (\{A(a)\},C(a)) $ 
 is  \emph{concept saturated for \T} if   concept saturation for \T is applied exhaustively. 
Similarly,
role saturation for \T consists of updating $C$ with the result $C'$ of  
  choosing an edge $ (\nu,\nu') $ in $ T_C $ and replacing the role name $r$ in the label 
 of $ (\nu,\nu') $ by a distinct $s\in\Sigma_\T$ such that $\T\models s\sqsubseteq r$
    if $ (\{A(a)\},C'(a)) $ 
 is still a positive example. 
 A positive example $ (\{A(a)\},C(a)) $ 
 is  \emph{role saturated for \T} if   role saturation for \T is applied exhaustively. 
\begin{example}\upshape
Let the signature be $ \Sigma_\T = \{A,B,r,s\} $, the 
target be $ \T=\{B\sqsubseteq A, A\sqsubseteq \exists s.B, s\sqsubseteq r \} $, 
the fixed ABox be $ \fixedAbox = \{A(a) \} $, the hypothesis 
be $\hypothesisOntology= \{B\sqsubseteq A, s\sqsubseteq r \}$ and the received counterexample 
be $ (\fixedAbox, \exists r.A(a)) $. Note that $ (\T,\fixedAbox)\models  \exists r.A(a)$ 
and $ (\hypothesisOntology,\fixedAbox)\not\models  \exists r.A(a)$. After 
having concept saturated the counterexample, it becomes  $ (\fixedAbox, A \sqcap \exists r.(A\sqcap B)(a)) $,
 and after having role saturated we obtain 
 $ (\fixedAbox, A \sqcap \exists s.(A\sqcap B)(a)) $, changing the order does not make any difference 
 in this example.
 \hfill {\mbox{$\triangleleft$}}
\end{example} 
Sibling merging for \T updates $C$ with the result $C'$ of  
  choosing nodes $ \nu, \nu_1, \nu_2 $  in $ T_C $ 
   such that $ \nu_1 $ and $ \nu_2 $ are   $r$-successors 
 of $ \nu $ and merging them (the merged node is connected to the successors of 
 $\nu_1, \nu_2$ and the label of it is the union of the labels of $\nu_1, \nu_2$) 
    if $ (\{A(a)\},C'(a)) $ 
 is still a positive example. 
 A positive example $ (\{A(a)\},C(a)) $ 
 is  \emph{sibling merged for \T} if   sibling merging for \T is applied exhaustively. 
We now define decomposition on the right for \T. 
Let $C_\nu$ be the concept corresponding to the sub-tree rooted in $\nu$ in $T_C$ and let
$C|^-_{\nu\downarrow}$ be the concept corresponding to the result of removing the
sub-tree rooted in $\nu$ from $T_C$. 
If $\nu'$ is an $r$-successor of $\nu$ in $T_C$, $A'$ is in the node label of $\nu$,
and   $ (\{A'(a)\}, \exists r.C_{\nu'}(a))$ is a positive example 
{plus $A'\not\equiv_{\Tmc}A$ if $\nu$ is the
root of $C$}, then replace $(\{A(a)\},C(a))$ by
\begin{enumerate}[label=(\alph*),leftmargin=*]
\item  $(\{A'(a)\}, \exists r.C_{\nu'}(a))$ if $\mathcal{H}\not\models A'\sqsubseteq \exists r.C_{\nu'}$; or 
\item  $(\{A(a)\}, C|^-_{\nu'\downarrow}(a))$, otherwise. 
\end{enumerate}  
A CI $A\sqsubseteq C$ is concept saturated/role saturated/sibling merged/decomposed
on the right for \T, that is, it is \T-essential, if this is the case for $ (\{A(a)\},C(a)) $.

\begin{restatable}{lemma}{polynomialtessential}\label{lem:size} 
	Let $ A\sqsubseteq C $ be  \tessential, then $|C|$ $\leq |\Sigma_\T||\T|$. 
\end{restatable}

Lemma~\ref{lem:size} is an easy consequence of~\cite[Lemma~2]{DBLP:conf/kr/DuarteKO18}. 
The only difference is that our ontologies may also have RIs, so our notion of a \T-essential CI includes role saturation (as in~\cite{KLOW18}).

 \begin{algorithm}[]
	\label{a:learningIQ}
	
	\SetAlgoLined
	\LinesNumbered
	\DontPrintSemicolon
	\KwIn{Signature $ \Sigma_\Tmc $}
	\KwOut{Hypothesis \hypothesisOntology}
	Compute \hypothesisOntology such that $ (\T,\fixedAbox) \equiv_{\AQ} (\hypothesisOntology,\fixedAbox) $ \;\label{a:learningIQ:AQ} 
	\While{$ (\T,\fixedAbox) \not\equiv_{\IQ} (\hypothesisOntology,\fixedAbox) $}{\label{a:learningIQ:loop1}
		Let $(\fixedAbox,C(a))$ be the positive counterexample returned by the oracle\; \label{a:learningIQ:ex}
		Let $\Amc=\{\alpha \in  \AQ\mid (\hypothesisOntology,\fixedAbox)\models \alpha \}$\; \label{a:learningIQ:sat}
		$\alpha =$ ReduceCounterexample($(\fixedAbox\cup\Amc,C(a)),\Hmc$)\; \label{a:learningIQ:reduce}
		Compute a \Tmc-essential CI $A\sqsubseteq D$ from $\alpha$\; \label{a:learningIQ:essential}
		\If{\text{there is} $A\sqsubseteq D'\in\hypothesisOntology$}{
		Find a \Tmc-essential $A\sqsubseteq D^\ast$ such that $\emptyset\models D^\ast\sqsubseteq D\sqcap D'$\; \label{a:learningIQ:t-essential}
		Replace $A\sqsubseteq D\in \hypothesisOntology$ by $A\sqsubseteq D^\ast$ \; \label{a:learningIQ:replace}
		}\Else{
		Add $A\sqsubseteq D$ to $\hypothesisOntology$
		}	
	}\label{a:learningIQ:loop2}
	\Return \hypothesisOntology\;	
	\caption{\ELH \IQ learning algorithm}
\end{algorithm}

To show that  Algorithm~\ref{a:learningIQ} runs in polynomial time (where
each call to an oracle counts as one step of computation),
we show that 
(1) the number of iterations is 
 polynomially bounded; and 
(2) each iteration can be computed in polynomial time. 
We start by arguing that each iteration 
requires polynomially many steps (Point~2). 
Line~\ref{a:learningIQ:reduce} of Algorithm~\ref{a:learningIQ}
can be computed in polynomial time using the same 
algorithm as the one used to learn DL-Lite$^\exists_\Rmc$ ontologies~\cite{DBLP:conf/aaai/KonevOW16}, except that 
here we do not use parent-child merging because \ELH does not have inverse roles. 
Line~\ref{a:learningIQ:sat} ensures that one can indeed find a singleton 
ABox, even though \ELH allows CIs of the form  $C\sqsubseteq A$.
For convenience of the reader we provide here the algorithm (Algorithm~\ref{a:ReduceCounterExample}) for refining counterexamples. 
Algorithm~\ref{a:ReduceCounterExample}  `walks inside' \A and $ C $ in order to find a singleton ABox which together with \T entails  a subconcept of $ C $. 
\begin{example}\upshape Assume $ \T=\{A\sqsubseteq \exists r.D \} $, $ \hypothesisOntology = \emptyset $, and
 $ \fixedAbox=\{r(a,b), A(b) \} $. Let  $ (\fixedAbox,\exists r.\exists r.D(a)) $ be the 
 positive counterexample received at Line \ref{a:learningIQ:ex} of Algorithm \ref{a:learningIQ}. 
 In the next line it calls the function `ReduceCounterexample' (Algorithm \ref{a:ReduceCounterExample}), 
 with $ (\fixedAbox,\exists r.\exists r.D(a)) $  and \hypothesisOntology as input. 
 The function makes a recursive call with the positive counterexample $ (\{r(a,b), A(b) \},\exists r.D(b)) $ and \hypothesisOntology as input.
In the new function call, Algorithm~\ref{a:ReduceCounterExample}
finds the singleton ABox $\{A(b)\}$ which satisfies the condition in Line~\ref{a:reduce:singleton} (i.e.,  $ (\T,\{A(b) \})\models\exists r.D(b)$ 
and returns $A\sqsubseteq \exists r.D$.  
\hfill {\mbox{$\triangleleft$}}
\end{example}

 \begin{algorithm}[]
		\label{a:ReduceCounterExample}
		\SetAlgoLined
		\LinesNumbered
		\DontPrintSemicolon
		\KwIn{Example $(\A,C(a))$,  TBox $ \hypothesisOntology $}
		\KwOut{CI $A\sqsubseteq D$}
		Let $D=\exists r.C'$  be a top-level conjunct of $C$ such that 
		  $(\T,\A)\models \exists r.C'(a)$ and $(\hypothesisOntology,\A)\not\models \exists r.C'(a)$\; \label{a:reduce:conjunct}
			\If{there is $ s(a,b) \in \A$ such that $ \hypothesisOntology\models s \sqsubseteq r  $ and $ (\T,\A)\models C'(b) $}{\label{a:reduce:if}
				$A\sqsubseteq D$=ReduceCounterexample($(\A,C'(b))$,\hypothesisOntology)\;\label{a:reduce:r}
			}
			\Else{
				Find a singleton $\{A(c)\}\subseteq \A$ such that $(\T,\{A(c)\})\models  D(c)$\;\label{a:reduce:singleton}
			}
			\Return $A\sqsubseteq D$\;	
		\caption{ReduceCounterexample}
	\end{algorithm}
	
We first argue that Algorithm~\ref{a:ReduceCounterExample} is 
implementable. That is, in Line~\ref{a:reduce:conjunct}, one can indeed
assume the existence of  a top-level conjunct $D=\exists r.C'$ of $C$ such that 
		  $(\T,\A)\models \exists r.C'(a)$ and $(\hypothesisOntology,\A)\not\models \exists r.C'(a)$. 
		  This follows from Lemma~\ref{l:existsNotConjuctConceptPositiveExample}. 

\begin{lemma}
	\label{l:existsNotConjuctConceptPositiveExample}
	Let $ \T $ be an \ELH TBox and \A an ABox. If $ (\T,\A)\not\models C(a)$ with $ C=\sqcap_{i=1}^n D_i $ then there is  a $ D_i $ with $ i\in\{1,\cdots, n\} $ such that $ (\T,\A)\not\models D_i(a) $.
\end{lemma}
\begin{proof}
	If $ (\T,\A)\models D_i(a) $ for all $ i\in\{1\cdots n\} $ means that there is 
	a homomorphism from $ h_i: \tree{D_i}\rightarrow \I_{\T,\A} $ mapping $ \rho_{D_i} $ 
	to $ a\in \Delta^{\I_{\T,\A}} $. It trivially follows that there is a homomorphism 
	$ h: \tree{C}\rightarrow \I_{\T,\A} $ obtained by merging all $ h_i $ into a unique function. 
	By \refLemma{\ref{l:subsumptionHomomorphism}}, 
	we reach a contradiction because this would mean that $ (\T,\A)\models C(a) $.
\end{proof}

Also,  if there is no
 $ s(a,b) \in \A$ such that $ \hypothesisOntology\models s \sqsubseteq r  $ and $ (\T,\A)\models C'(b) $
 (that is, the condition in Line~\ref{a:reduce:if} is not satisfied)
then one can indeed find a singleton $\{A(c)\}\subseteq \A$ such that $(\T,\{A(c)\})\models  D(c)$ in 
Line~\ref{a:reduce:singleton}. This follows from Lemma~\ref{l:existsSingletonAbox}.
To prove Lemma~\ref{l:existsSingletonAbox}, we first show the following two technical lemmas. 
Let $ \I_C $ be an interpretation, 
the \emph{one-neigbourhood}  $N_{\I_C}(a)$ of $ a\in\Delta^{\I_C} $ 
is the set of concept names $ A $ with $ a\in A^{\I_C} $.

\begin{lemma}
	\label{l:existsconceptname}
	Let \T be an \ELH terminology and let $C$ and $ D=\exists r.D' $ be 
	concept expressions. Assume 
	there is a homomorphism $ h:\tree{D}\rightarrow \I_{\T,C} $  
	such that $ h(\rho_{D}) = d\in \Delta^{\I_C} $ and the image of the 
	subtree \tree{D'} of \tree{D} under $ h $ is included in 
	$ \Delta^{\I_{\T,C}} \setminus \Delta^{\I_C} $. 
	Then there exists $ A\in N_{\I_C}(d) $ such that $ \T\models A\sqsubseteq D $.
\end{lemma}
\begin{proof}
Since \T is an \ELH terminology, it contains only CIs of the form $A\sqsubseteq F$ 
or $F\sqsubseteq A$ with $F$ an \ELH concept expression and $A$ a concept name. 
In the construction of the canonical model (Definition~\ref{d:canonicalModel}), the definition of $\I_{n+1}$
is the result of either 
\begin{itemize}
\item adding a new element of $\Delta^{\I_n}$ to the extension 
$A^{\I_n}$ of a concept name $A$ (to satisfy  
a CI of the form $F\sqsubseteq A$); or 
\item identifying the root of the tree interpretation $\tree{F}$  of a concept expression $F$
 with some element of $\Delta^{\I_n}$ in the extension $A^{\I_n}$ of 
 a concept name $A$ (to satisfy  
a CI of the form $A\sqsubseteq F$). 
\end{itemize}
Thus, if there is a homomorphism $ h:\tree{D}\rightarrow \I_{\T,C} $  
	such that $ h(\rho_{D}) = d\in \Delta^{\I_C} $ and the image of the 
	subtree \tree{D'} of \tree{D} under $ h $ is included in 
	$ \Delta^{\I_{\T,C}} \setminus \Delta^{\I_C} $ 
	this can only be if there is $ A\in N_{\I_C}(d) $ such that $ \T\models A\sqsubseteq D $.	
\end{proof}

\begin{lemma}
	\label{l:ancestorsMappedABox}
	If there is a homomorphism $h: \tree{C}\rightarrow \I_{\T,\A}  $ 
	for an arbitrary concept expression $ C $, an \ELH terminology \T, an ABox \A, 
	and a node $ \nu $ in $ \tree{C} $ is mapped to $ a\in\Delta^{\I_\A} $, 
	then all ancestors of $ \nu $ are mapped into $\Delta^{\I_\A}  $.
\end{lemma}
\begin{proof}
	 \ELH does not allow inverse roles.
		So, if $\nu$ has a parent $\nu'$
	 it must be in $ \Delta^{\I_{\A} } $ because 
	 in the construction of canonical model $ \I_{\T,\A} $
only role successors are connected to the elements of $ \Delta^{\I_{\A} } $.
\end{proof}

\begin{lemma}
	\label{l:existsSingletonAbox}
	Let \T be the target, \hypothesisOntology the hypothesis and \A an ABox. If $ (\T,\A)\models \exists r.C(a) $, there is no $ s(a,b)\in\A $ such that $ \T\models s\sqsubseteq r $ and $ (\T,\A)\models C(b) $, 
	then there is a singleton ABox $\A'$ such that $ (\T,\A')\models \exists r.C(a) $. 
\end{lemma}
\begin{proof}
	Let $ C'=\exists r.C $, we know that there is a homomorphism $ h: \tree{C'}\rightarrow \I_{\T,\A}  $
	 such that $ h(\rho_{C'})=a\in\Delta^{\I_{\A} } $. If we show that all nodes 
	 in $ \tree{C} $ are mapped into $ \Delta^{\I_{\T,\A} }\setminus \Delta^{\I_{\A} } $ 
	 then we are able to conclude that there is a concept name $ A $ such that $ \T\models A\sqsubseteq C' $ by using \refLemma{\ref{l:existsconceptname}}.
	
	Let 
	 $ \nu $ a descendant of $ \rho_C $. 
	For a proof by contradiction we assume that a node $ \nu $ is 
	mapped by $ h $ to an element in $ \Delta^{\I_{\A}}  $. 
	By \refLemma{\ref{l:ancestorsMappedABox}}, all its ancestors 
	 must be mapped into $ \Delta^{\I_{\A} } $ by $ h $. 
	 Remember that $ h(\rho_{C'})=a $, this means that 
	 exists an $ s(a,h(\rho_C))\in\A $ such that 
	 $ \T\models s\sqsubseteq r $ and $ (\T,\A)\models C(h(\rho_C)) $, which contradicts our assumption.
\end{proof}

So we have that Algorithm~\ref{a:ReduceCounterExample} is implementable.  
 It is easy to see that if the input is a positive counterexample 
 then the output is also a positive counterexample (for \T and \hypothesisOntology). 
Each line of Algorithm~\ref{a:ReduceCounterExample}  can be computed 
using polynomially many steps. Regarding the recursive calls, we point out 
that, each time Algorithm~\ref{a:ReduceCounterExample} makes a recursive 
call, the example passed as parameter to the function is strictly smaller, 
so the number of recursive calls is bounded by the size of the counterexample 
received in Line \ref{a:learningIQ:ex}  
of Algorithm~\ref{a:learningIQ}.

Now we have that in Line~\ref{a:learningIQ:reduce} of Algorithm~\ref{a:learningIQ}
we obtain a positive counterexample of the form $(\{A(a)\},C(a))$ 
in polynomial time. 
It follows from 
the proof of Lemma~2 in~\cite{DBLP:conf/kr/DuarteKO18} that 
 Line~\ref{a:learningIQ:essential} can also be computed in polynomial time, and moreover, 
 the size of resulting CI is polynomial in $|\T|$. Indeed
  Lemma~\ref{lem:size} is an easy consequence of~\cite[Lemma~2]{DBLP:conf/kr/DuarteKO18} and 
the only difference in the setting of Lemma~2 and ours is that \ELH allows RIs. We observe that our notion of   \T-essential  includes role saturation (as in~\cite{KLOW18}). 
That is, a positive example $ (\A,C(a))$ is \tessential if 
it is concept saturated, role saturated, sibling merged and 
decomposed on the right for \T. This notion is also used for CIs. 
A CI $A\sqsubseteq C$ is \tessential 
if this is the case for $ (\{A(a)\},C(a)) $. 
Role saturation for \T can be implemented in polynomial time~\cite{KLOW18} and the result 
of Lemma~2 in \cite{DBLP:conf/kr/DuarteKO18} can be easily extended to our case.

The fact that Line~\ref{a:learningIQ:t-essential} is also implementable in polynomial time follows from Lemma~\ref{l:tessentialCombination} 
below, which is an adaptation of~\cite[Lemma~7 used to show Theorem~2]{DBLP:conf/kr/DuarteKO18} (see also~\cite[Lemma~30]{KLOW18}).

\begin{lemma}
	\label{l:tessentialCombination}
	Let $ A\sqsubseteq C_1 $ and $A\sqsubseteq C_2  $ 
	be  \tessential positive examples. 
	One can construct a \tessential $ A\sqsubseteq C $ 
	such that $ \emptyset\models C\sqsubseteq C_1\sqcap C_2 $ in polynomial time  
	in $ |C_1| + |C_2|$. 
\end{lemma}
 \begin{proof}[Sketch]
The main difference is that here the ontology language is \ELH, which also includes RIs. 
Since $ A\sqsubseteq C_1 $ and $A\sqsubseteq C_2  $  are \tessential, these examples 
are role saturated. By assumption, the target ontology does not entail non-trivial 
RIs (see begin of paragraph `Learning \ELH ontologies with \IQ' in Section~\ref{sec:learning-algorithms}).
Thus, the presence of RIs does not affect the application of  sibling merging. 
 \end{proof}

The rest of this subsection is devoted to show that 
 the number of iterations is polynomial in the size of \T. In each iteration either a CI is replaced or it is added 
 to the hypothesis. The number of times it is added to the hypothesis is bounded by $|\Sigma_\Tmc|$. 
Thus, it remains to show that the number of replacements is also bounded polynomially in $|\Tmc|$.

To show this, we use Lemmas~\ref{l:isomorphicEmbedding} and~\ref{l:tessentialCombinationNodes} below. 
Given the tree representation $\tree{C}$ of a concept $C$ and an interpretation \Imc, 
we say that a homomorphism $h: \tree{C} \rightarrow \I$ is 
an \emph{isomorphic embedding for \T} if it is injective, $A\in l(\nu)$ if
$h(\nu)\in A^{\Imc}$ for all concept names $A$, and for $r=l(\nu,\nu')$ it
holds that $\Tmc\models r\sqsubseteq s$ for all
$(h(\nu),h(\nu'))\in s^{\Imc}$.
\begin{lemma}[Isomorphic Embedding]
	\label{l:isomorphicEmbedding}
	Let 
	$ A\sqsubseteq C$
	 be  a \tessential CI. If $\T \models A \sqsubseteq D$ and  
	$\T \models D \sqsubseteq C$
	then any homomorphism $h: \tree{C} \rightarrow \I_{D,\T}$ such that $h(\rho_C) = \rho_D$ is an isomorphic embedding for \T. 
\end{lemma}
\begin{proof}[Sketch] The argument is as in~\cite[Lemma~31]{KLOW18}. Non-injectivity would contradict that 
 $A\sqsubseteq C$ is sibling merged (which follows from the assumption that it is \tessential).
 The remaining conditions for the notion of  isomorphic embedding for \T follow 
 from the fact that  $ A\sqsubseteq C$ is concept and role saturated for \T  (which 
 again follow from the assumption that the example is \tessential).
\end{proof}

By the Claim of Lemma~33 in~\cite{KLOW18} the following holds. 
\begin{lemma}
	\label{l:tessentialCombinationNodes}
	If  $ A\sqsubseteq C$ is \tessential,
	 $ \T\models A\sqsubseteq C' $ and $ \emptyset\models C'\sqsubseteq C $
	  and $ \emptyset\not\models C\sqsubseteq C' $, then $ \tree{C} $ is obtained from $ \tree{C'} $ by removing at least one subtree.
\end{lemma}

We are now ready for Lemma~\ref{lem:number}, which bounds the number of replacements. 

 \begin{restatable}{lemma}{polynomialiteration}\label{lem:number} 
	For every $A\in\Sigma_\T$, the number of replacements of a CI in 
	\hypothesisOntology in Algorithm~\ref{a:learningIQ} is bounded polynomially in $|\T|$. 
\end{restatable}
\begin{proof}
All CIs that are added or replace a CI in \hypothesisOntology are \tessential and, therefore, their 
size is polynomially bounded by $|\T|$ (Lemma~\ref{lem:size}).
As already argued in~\cite[Lemma~33]{KLOW18}, 
when $A \sqsubseteq C$ is replaced with $A\sqsubseteq C'$, then
  $\emptyset \models C'\sqsubseteq C$ and $\emptyset \not\models C
  \sqsubseteq C'$ (otherwise the positive counterexample returned by
  the oracle would be a consequence of $\hypothesisOntology$). Moreover, both $A \sqsubseteq C$ and
  $A\sqsubseteq C'$ are consequences of \Tmc. 
  So the conditions of
 Lemma~\ref{l:tessentialCombinationNodes} are satisfied, and so, whenever 
 a CI $A\sqsubseteq C$   is replaced by by some other CI $A\sqsubseteq C'$, 
 we have that $\treeV{C'}>\treeV{C}$.  
 
 The presence of CIs of the form $C\sqsubseteq A$ in our setting 
 does not affect the argument in~\cite[Lemma~33]{KLOW18} because the CIs we mention are \tessential, in particular, they 
 are concept saturated for \T, and so, they always satisfy the inclusions of the form $C\sqsubseteq A$ .
\end{proof}

This concludes our proof of Theorem~\ref{thm:iqupper}. 

\subsection{Learning \ELH ontologies with \CQr}
We show that one can convert in polynomial time 
any $ q\in$ \CQr 
in a positive counterexample of the form $ (\fixedAbox,q) $
into an \IQ $q'$ such that $ (\fixedAbox,q') $ 
is a positive counterexample. 
Then our upper bound 
 follows from Theorem~\ref{thm:iqupper}.
To transform the query in a positive counterexample into an \IQ, we use the following three operations.  
\begin{enumerate}
\item 
Given a positive counterexample $(\fixedAbox,q)$,
\emph{individual saturation for \T} consists of updating $q=\exists \vec{x}\varphi(\vec{a},\vec{x}) $ with the result $q'$ of  
   choosing $x \in\vec{x}$   and   $a\in\individuals{\fixedAbox}$ and replacing 
   $x$ by  $a$    if $ (\fixedAbox,q') $ 
 is still a positive counterexample. 
\item Given a positive counterexample $(\fixedAbox,q)$,
\emph{merging for \T} consists of updating $q=\exists \vec{x}\varphi(\vec{a},\vec{x}) $ with the result $q'$ of  
   choosing distinct $x, x' \in\vec{x}$     and replacing 
   all occurrences of $x$ by  $x'$    if $ (\fixedAbox,q') $ 
 is still a positive  counterexample. 

\item Given a positive counterexample $(\fixedAbox,q)$, 
\emph{query role saturation for \T} consists of updating $q=\exists \vec{x}\varphi(\vec{a},\vec{x}) $ with the result 
$ q' $ of choosing an atom $s(t,t') $ and a role $r\in\Sigma_\T$ 
such that $ \T\models r \sqsubseteq s $ (under the assumption of not having equivalent roles, as described above for learning with \IQs) 
and replacing $s(t,t') $ by $r(t,t') $ 
if $ (\fixedAbox,q') $ 
is still a positive  counterexample. 
\end{enumerate}
We say that a positive counterexample $ (\fixedAbox,q) $ 
 is  \emph{individual saturated}/ \emph{merged}/  \emph{query role saturated for \T} if  
  individual saturation/ merging/ query role saturation for \T has been exhaustively applied.
Given a query $q\in \CQr$, we denote by $G^x_q$ the induced subgraph of $G_q=(V,E)$  
that has as nodes $x$, where $x\in V$, and nodes in $V$ that are reachable from $x$ via a directed path.
We say that $G^x_q$ is \emph{tree shaped} if  there is a unique element (the root), 
denoted by $\rho_G$, such that 
(i) 
for every 
node $d$ there is 
a directed path from $\rho_G$ to $d$ and (ii) 
for every distinct directed paths $p_1,p_2 $ starting from $\rho_G$ 
their last elements are distinct.

\begin{lemma}\label{l:anonymoustree}
Let \T and \hypothesisOntology  be \ELH TBoxes and assume \T and \hypothesisOntology entail the same RIs.
	If a positive counterexample $ (\fixedAbox,q) $ (for \T and \hypothesisOntology), 
	with $ q\in \CQr $, is individual saturated/ merged and query role saturated 
	for \T, then 
	 for all $ x $ occurring in $ q $, the graph $G^x_q$ is  	tree shaped 
	 and only one individual name reaches $ x $ with a path that visits only variables. 
\end{lemma}
\begin{proof}[Sketch]
	Since $ (\T,\fixedAbox)\models q $, there is a homomorphism $ h $ from 
	$ q =\exists \vec{x}\varphi(\vec{a},\vec{x})$ to $ \I_{\T,\fixedAbox} $ mapping every individual to itself.
	Since individual saturation has been exhaustively applied to $ q $,
	 every $ x\in\vec{x} $ is mapped by $ h $ into  $ \Delta^{\I_{\T,\fixedAbox}}\setminus  \Delta^{\I_{\fixedAbox}}$. If this is not the case a variable is mapped by $ h $ into $\Delta^{\I_{\fixedAbox}}$ and individual saturation is not exhaustively applied, reaching a contradiction. 
	
	Suppose that  $G^x_q$ is not tree-shaped. Since  \T and \hypothesisOntology entail the same RIs and 
	query role saturation for \T has been applied it is not the case that 
	for $ r,s\in\Sigma_\T $ there is a variable that is both $ r $-successor and $ s $-successor of another variable.
	Therefore, there is an undirected cycle in $ G^x_{q} $ with three or more nodes. This means that for $ w,y,z \in \vec{x} $ with $ w\neq y\neq z $, 
	 atoms of the form $ r(w,y) $, $ s(z,y) $ are in $ q $.
	We know  that $ h(y) $ is mapped into $ \Delta^{\I_{\T,\fixedAbox}}\setminus  \Delta^{\I_{\fixedAbox}}$, and 
	by the construction of the canonical model $ \I_{\T,\fixedAbox} $,
	it follows that $ h(y) $ has only one parent.
	Thus $ h(w)=h(z) $. 
	But this means that every occurrence of $ w $ can be replaced by $ z $. 
	This contradicts the fact that $(\fixedAbox,q)$  has been merged for \T. 
	Therefore $ G^x_{q} $ is  tree-shaped.
	The fact that, for all $x\in\vec{x}$, only one individual name reaches $ x $ with a path that visits only variables
	follows from the structure of the canonical model $ \I_{\T,\fixedAbox} $ and the fact that $q$ is individual saturated for \T.
\end{proof}

\lemmalearningCQr*
\begin{proof}
	We use the following claim.  

\begin{claim}
	\label{cl:runningrules}
	Given a positive counterexample $ (\fixedAbox,q) $ with $ q\in\CQr $. 
	One can compute a positive counterexample $ (\fixedAbox,q') $ that 
	is individual saturated/ merged/ query role saturated for \T in polynomial time with respect to $ | (\fixedAbox,q) |$ and $ |\Sigma_\T| $.
\end{claim}
	We prove this claim by analysing the running time of individual saturation, merging and query role saturation for \T.
	There are at most $ |q| $ variables in $ q $ that can potentially be replaced with individuals. Since the number of individuals is bounded by $ |\fixedAbox| $, after at most  $ |\fixedAbox||q| $ membership queries, $ (\fixedAbox,q) $ is individual saturated for \T.
	There are at most $ |q| $ variables in $ q $ that can be merged. Therefore, 
	after at most $ |q|^2 $ membership queries $ (\fixedAbox,q) $ is merged for \T.
	We assume without loss of generality that role names have a representative, so no equivalent role name is in $ q $.
	There are at most $ |q| $ atoms of the form $ s(t,t') $ with $ t,t'\in\vec{a}\cup\vec{x} $ in $ q $ and at most $ |\Sigma_\T| $ different role names 
	$r$ such that $ \T\models r \sqsubseteq s $. 
	For every $  s(t,t') $ in $ q $, it is checked in polynomial time if  the query that results from replacing  $  s(t,t') $ by $r(t,t')$ 
	in $ q $ is still entailed by $(\T,\fixedAbox)$ (that is, after modifying $q$, whether $ (\fixedAbox,q) $ is still a positive counterexample).
	After at most $ |q||\Sigma_\T| $ membership queries $ (\fixedAbox,q) $ is query role saturated for \T.
This finishes the proof of this claim. 

\smallskip

	Given a positive counterexample $ (\fixedAbox,q) $ with $q\in$ \CQr, let 
	$ (\fixedAbox,q') $ be the result of exhaustively applying 
	individual saturation/ merging/ query role saturation for \T. 
	Then, 
	 for all $ x $ occurring in $ q' $, the graph $G^x_{q'}$ is  	tree shaped 
	 and only one individual name reaches $ x $ with a path that visits only variables. 
	 If $G^x_{q'}$ is tree shaped, then one can translate the subquery of $q'$ 
	  containing  $x$ and other variables reachable from $x$ into 
	  a concept expression, denoted $C^x_{{q'}}$.  
	  For every  atom 
	 of the form  $A(a)$  in $ q' $ we check whether 
	 $ (\hypothesisOntology,\fixedAbox)\not\models A(a)$ (atoms of the form $r(a,b)$ 
	 do not need to be checked since by assumption \T and \Hmc entail the same RIs). 
	If this is the case for any such $ A(a) $ we have that $(\fixedAbox,A(a))$ 
	is a positive counterexample and we are done. 
	
	Otherwise we claim that there is $r\in\Sigma_\T$, $a\in\individuals{\fixedAbox}$, and 
	a variable $x$ in $q'$ such that 
	$(\fixedAbox,\exists r.C^x_{{q'}}(a))$ is a positive counterexample. 
	Indeed, suppose to the contrary that for all 
	$r\in\Sigma_\T$, $a\in\individuals{\fixedAbox}$, and 
	 $x$ occurring in $q'$, $(\T,\fixedAbox)\models\exists r.C^x_{{q'}}(a)$ 
	 implies $(\Hmc,\fixedAbox)\models\exists r.C^x_{{q'}}(a)$. 
	 Then there are homomorphisms $h_{\exists r.C^x_{{q'}}(a)}$ from all such $T_{\exists r.C^x_{{q'}}}$
	 into $\Imc_{\hypothesisOntology,\fixedAbox}$  mapping $a$ to itself. 
	 By assumption, for every  atom 
	$ \alpha $ of the form  $A(a)$ or $r(a,b)$ in $ q' $, we have that 
	 $ (\hypothesisOntology,\fixedAbox)\models \alpha$. Thus, 
	 one can construct a homomorphism from $q'$ to $\Imc_{\hypothesisOntology,\fixedAbox}$
	 by taking the union $h$ of all such $h_{\exists r.C^x_{{q'}}(a)}$ 
	 and extending $h$ by mapping all individual names occurring in $q'$ 
	 into themselves.	 
	 We thus have that $ (\hypothesisOntology,\fixedAbox)\models q'$, 
	 which contradicts the fact that $ (\fixedAbox,q') $ is a positive counterexample.
\end{proof}

\begin{lemma}\label{lem:iqcq}
Let $\T$ and \Hmc be \ELH TBoxes which entail the same RIs and let $\fixedAbox$ be an ABox.  
If $(\T,\fixedAbox)\equiv_{\IQ} (\Hmc,\fixedAbox)$ then $(\T,\fixedAbox)\equiv_{\CQr} (\Hmc,\fixedAbox)$.
\end{lemma}
\begin{proof}[Sketch]
Assume to the contrary that there is a query $q\in\CQr$ such that 
$(\T,\fixedAbox)\models q$ but $(\Hmc,\fixedAbox)\not\models q$ 
(or vice-versa).  By Lemma~\ref{lemma:cqrupper}, 
	 one can contruct a positive counterexample $ (\fixedAbox,q') $ with 
	 $q'\in\IQ$ in polynomial time 
	 in  $|(\fixedAbox,q)||\Sigma_\T|$. 
	 This contradicts the assumption that $(\T,\fixedAbox)\equiv_{\IQ} (\Hmc,\fixedAbox)$.
	 Thus, $(\T,\fixedAbox)\equiv_{\CQr} (\Hmc,\fixedAbox)$.
\end{proof}

\learningCQr*
\begin{proof}
	The target \T is a terminology, therefore it contains CIs of the form $ C\sqsubseteq A $ or $ A\sqsubseteq C $. All needed CIs of the form $ C\sqsubseteq A $ can be learned without asking any inseparability query according to Theorem~\ref{thm:aq}.
	Moreover, all RIs $r\sqsubseteq s$ can be learned in polynomial time using membership queries of the form $(\{r(a,b)\},s(a,b))$. 
	It remains to learn CIs of the form $ A\sqsubseteq C $ after the learner receives a positive counterexample $ (\fixedAbox,q) $ with $ q\in\CQr $.
	By Lemma~\ref{lemma:cqrupper} we can find a positive counterexample of the form $ (\fixedAbox,C(a)) $ 
	from a positive counterexample $ (\fixedAbox, q) $, with $q\in\CQr$, in polynomial time.
	Therefore, by Theorem~\ref{thm:iqupper} we can learn a $\IQ$-inseparable hypothesis in $\PTimeL$. 
	If \ELH TBoxes entail the same RIs and are $\IQ$-inseparable then, by Lemma~\ref{lem:iqcq}, they are $\CQr$-inseparable  as well (w.r.t. some 
	ABox $\fixedAbox$).
\end{proof}

\section{Proofs for Section ``Data Updates''}

\updateinseparabilitybisimulation*
\begin{proof}

	For all concept expressions $ C $, it holds by Lemma~\ref{lemmabisimulation} that $ (\hypothesisOntology,\A)\models C(b) $ iff $ (\hypothesisOntology,\fixedAbox)\models C(a) $.
	Since we assumed $ (\T,\fixedAbox) \equiv_{\IQ} (\hypothesisOntology,\fixedAbox) $, it holds that $ (\hypothesisOntology,\fixedAbox)\models C(a) $ iff $ (\T,\fixedAbox)\models C(a) $. But again by Lemma~\ref{lemmabisimulation}, $ (\T,\fixedAbox)\models C(a) $ iff $ (\T,\A)\models C(b) $. 
	Since we assumed that \T and \hypothesisOntology entail the same RIs the statement holds.
\end{proof}

\learningpoltimegeneralize*
\begin{proof}[Sketch] 
We first focus on CIs of the form $C\sqsubseteq A$, with $A\in\NC$. 
	Let \Hmc be a hypothesis with CIs computed with Algorithm~\ref{a:learningAQ} 
	(modified to ask only membership queries, cf. Theorem~\ref{thm:aq}) and then generalised.
	Clearly, 
	 \Hmc can be generalised for \T   in polynomial time in $|\Sigma_\T|$ and $|\Hmc|$. 
		\begin{claimt}
		\label{claimlinearreplacement}
		For all $ \A\in {\sf g}_\T(\fixedAbox) $ and all $ \alpha \in \A $, 
		 $ (\Hmc,\fixedAbox)\models \alpha $. 
		 Moreover, if there is a homomorphism $h:T_C\rightarrow \Imc_{\Hmc,\fixedAbox}$ mapping $\rho_C$ to some $a\in\individuals{\fixedAbox}$, 
		 where $C\sqsubseteq B\in\Hmc$, then there is a homomorphism $h:T_C\rightarrow \Imc_{\Hmc,\A}$  mapping $\rho_C$  to $a$.
	\end{claimt}
		By  definition of ${\sf g}_\T(\fixedAbox)$, $ \A $ is a subset of  $\fixedAbox\cup\{\alpha\in\AQ\mid (\Tmc,\fixedAbox)\models \alpha\}$. 
		Therefore, since $(\Hmc,\fixedAbox) \equiv_{\IQ} (\Tmc,\fixedAbox)$, for  every assertion  $\alpha \in \A $,  $ (\Hmc,\fixedAbox)\models \alpha $. 
	The second statement follows from the fact that \Hmc is generalised for \T 
	and the definition of $ {\sf g}_\T(\fixedAbox) $.

	\smallskip

	Let $ \A \in {\sf g}_\T(\fixedAbox) $.
Assume $(\T,\Amc)\models A(a)$. By definition of \A, 
$(\T,\fixedAbox)\models A(a)$. As $(\Hmc,\fixedAbox) \equiv_{\AQ} (\Tmc,\fixedAbox)$, 
$(\Hmc,\fixedAbox)\models A(a)$. 
		If $(\Hmc,\fixedAbox)\models A(a)$, with $A\in\NC$, and 
	$A(a)\not\in\fixedAbox$ then there is $C\sqsubseteq A\in\Hmc$ such that 
	$a\in C^{\Imc_{\Hmc,\fixedAbox}}$.
	By Claim~\ref{claimlinearreplacement} and Lemma~\ref{l:homomorhelp},
	if $A(a)\not\in\fixedAbox$ then $a\in C^{\Imc_{\Hmc,\Amc}}$. 
 Thus, by the first statement of Claim~\ref{claimlinearreplacement}, $(\Hmc,\fixedAbox)\models A(a)$.
 As $\T\models\Hmc$ (c.f. Lemma~\ref{lem:polyexample} and the definition of generalisation for \T), if $(\Hmc,\Amc)\models A(a)$ then $(\T,\Amc)\models A(a)$. 
	
	Thus, after computing a generalised \Hmc for \T we obtain a TBox 
	that is not only \AQ-inseparable from \T w.r.t. $\fixedAbox$ but 
	also w.r.t. all  $ \A\in {\sf g}_\T(\fixedAbox) $. 
	CIs of the form $A\sqsubseteq D$, with $A\in\NC$, are not affected by the ABox updates since 
$\Sigma_\T\subseteq \Sigma_{\fixedAbox}$ and  we can use Algorithm~\ref{a:learningIQ} 
 with counterexamples of the form $(\Amc,C(a))$ with $ \A \in {\sf g}_\T(\fixedAbox) $ in the same way
 as with $\Amc=\fixedAbox$. All RIs can be easily learned with 
membership queries of the form 
	 $(\{r(a,b)\},s(a,b))$.
\end{proof}

  \section{Proofs for Section ``Learning from Data''}

\existsbatch*

\begin{proof}

Let $\queryLanguage = \AQ$.
Algorithm~\ref{a:learningAQ}, given a positive counterexample $(\fixedAbox, A(a))$,
calls Algorithm \ref{a:treeShape} on it,
returning an example $(\Amc, B(b))$ (where \A is tree shaped ABox 
encoding a concept $C_\A$) such that there exists an ABox homomorphism from $\Amc$ to $\fixedAbox$.
Let $S$ be the set of all such examples, and consider the batch defined as follows:
\begin{align*}
	\batch  = & \ S \ \cup \\
				  & \{ (\{ A(a) \}, B(a)) \mid \Tmc \models A \sqsubseteq B, \text{ with } A, B \in \Sigma_\Tmc \} \ \cup \\
	& \{ (\{ r(a, b) \}, s(a, b) ) \mid \Tmc \models r \sqsubseteq s \text{ with } r,s\in\Sigma_\Tmc \}
\end{align*}
Given $\batch$ as input, which is polynomial in $|\Tmc|$,
we can construct $\hypothesisOntology \in \hypothesisSpace$ by setting:
\begin{align*}
	\hypothesisOntology  := & \ \{ C_{\Amc} \sqsubseteq B \mid (\Amc, B(b)) \in S \} \ \cup \\
	& \ \{ A \sqsubseteq B \mid (\{ A(a) \}, B(a)) \in \batch \} \ \cup \\
& \ \{ r \sqsubseteq s \mid (\{ r(a, b) \}, s(a, b) ) \in \batch\},
\end{align*}
so that $(\hypothesisOntology, \fixedAbox) \equiv_{\AQ} (\targetOntology, \fixedAbox)$.
For $\queryLanguage = \IQ$, let $S'$ be the set of all the polynomially many (in $|\Tmc|$)
 positive counterexamples $(\fixedAbox, C(a))$ considered in Algorithm~\ref{a:learningIQ}.
We extend the batch $\batch$, as defined above, to another batch $\batch' = \batch \cup S'$.
We can define an ontology $\hypothesisOntology' \in \hypothesisSpace$ by adding to $\hypothesisOntology$ (defined as above) the CIs obtained by applying the procedure in  the while-loop of Algorithm~\ref{a:learningIQ} on the examples in $S'$.
Clearly, $(\hypothesisOntology', \fixedAbox) \equiv_{\IQ} (\targetOntology, \fixedAbox)$.
Since by assumption $\Sigma_\T\subseteq \Sigma_{\fixedAbox}$, we can add to the batch 
examples that allow a learning algorithm to learn all the RIs.
Then, by Lemma~\ref{lem:iqcq}, we also have that $(\hypothesisOntology', \fixedAbox) \equiv_{\CQr} (\targetOntology, \fixedAbox)$.
\end{proof}

\paclearn*
\begin{proof}
The proof slightly generalises the one presented in~\cite[Th. 13.3]{MohEtAl}, in that we allow the probability distribution on a subset $\set \subseteq \examples$, and is reported here for the convenience of the reader.

Let $\Fmf = (\examples, \set, \hypothesisSpace, \mu)$ be a learning framework that is polynomial time exact learnable.
In the execution of the algorithm used to learn $\Fmf$, we replace each equivalence query with respect to $\set$ by a suitable number of calls to $\EX$, given a distribution $\prob$.
If all the classified examples returned by $\EX$ are consistent with the hypothesis $h$, 
then the algorithm proceeds as if the equivalence query with respect to $\set$ had returned `yes'. 
Otherwise, there is a classified example $(e,\lab{e}{t})$ such that either
$e \in \mu(h)$ and $\lab{e}{t}=0$, or $e \not\in \mu(h)$ and $\lab{e}{t}=1$. 
In this case, the algorithm proceeds as if the equivalence query with respect to $\set$ had returned `no', with $e$ as a counterexample.
Consider the $i$-th equivalence query as the $i$-th stage of the algorithm, 
and assume that at stage $i$ the algorithm makes $m_{i}$ calls to $\EX$ in place of the equivalence query.

Fix $\epsilon, \delta \in (0, 1)$ and 
let $$\sample_{m_{i}}  = \{ (e_{1},\lab{e_{1}}{t}), \ldots, (e_{m_i},\lab{e_{m_i}}{t})  \}$$ be 
a \emph{sample} of size $m_i$ generated by $\EX$ at stage $i$. Denote by 
$\sample^{-\ell_{t}}_{m_{i}}$ the set of elements of $\set$ occurring in $\sample_{m_{i}}$.
As the examples in $\set$ are assumed to be
independently and identically distributed
according to $\prob$, 
the probability distribution \prob induces $ \prob'$, with  
$$\prob'( \{ \sample^{-\ell_{t}}_{m_{i}}  \}):=\prod_{j = 1}^{m_{i}} \prob(\{ e_{j}\}).$$ 
As  described, the algorithm will return a hypothesis such that  all classified examples returned by $\EX$ 
at some stage $i$ are consistent with the hypothesis. In symbols:
$$(\sample^{-\ell_t}_{m_i}\cap \mu(h)) =(\sample^{-\ell_t}_{m_i}\cap \mu(\target)).$$
We write $ \con{ \sample_{m_{i}}}{h}$ as an abbreviation for the 
 above. 
Fix some $h\in\hypothesisSpace$ such that $\prob((\mu(h) \oplus \mu(\target)) \cap \set) > \epsilon$. Those 
hypothesis are called `bad' because the error is larger than $\epsilon$. 
We want to bound the probability of finding a bad hypothesis. 
Let $n$ 
be the total number of equivalence queries needed by the algorithm (which exists by assumption) to exactly learn $\Fmf$ in polynomial time.
For all $t,h\in\hypothesisSpace$, we know that $\prob((\mu(h) \oplus \mu(\target)) \cap \set) > \epsilon$ iff
 $\prob( \set \setminus(\mu(h) \oplus \mu(\target)) ) \leq (1 - \epsilon)$.
Then we have the following: 
\begin{align*}
& \prob' (\bigcup^n_{i=1} \{ \sample^{-\ell_t}_{m_{i}}\mid \con{ \sample_{m_{i}}}{h}
 \} )
\leq 
\sum_{i = 1}^{n} \prob' ( \{ \sample^{-\ell_t}_{m_i} \mid \con{ \sample_{m_{i}}}{h} \} ) \leq \\
& \sum_{i = 1}^{n} \prod_{j = 1}^{m_{i}} \prob(\{ e_{j} \mid e_{j}  \in \sample^{-\ell_t}_{m_i},\con{ \sample_{m_{i}}}{h}  \}) \leq
\sum_{i = 1}^{n} (1 - \epsilon)^{m_{i}}.
\end{align*}

The latter   is bounded by $\delta$ for $m_{i} \geq \frac{1}{\epsilon} (\ln\frac{1}{\delta} + i\ln2)$.
Indeed, we have:

\begin{align*}
& \sum_{i = 1}^{n} (1 - \epsilon)^{m_{i}} \leq
\sum_{i = 1}^{n} e^{-\epsilon m_{i}} \leq
\sum_{i = 1}^{n} e^{-\epsilon \frac{1}{\epsilon} (\ln\frac{1}{\delta} + i\ln2)} \leq \\
& \sum_{i = 1}^{n} e^{-(\ln\frac{1}{\delta} + i\ln2)} \leq
\sum_{i = 1}^{n} e^{-(\ln\frac{2^{i}}{\delta})} \leq \\
&\sum_{i = 1}^{n} e^{\ln\frac{\delta}{2^{i}}} \leq 
\sum_{i = 1}^{n} \frac{\delta}{2^{i}} \leq \delta.
\end{align*}

Since we assumed that the original algorithm learns $\Fmf$ is polynomial time, $i$ is 
polynomial in $|\target|$ and $|e|$, where $e$ is the largest example in $\sample_{m_{n}}$.
Therefore, $\Fmf$ is polynomial time PAC learnable with membership queries.
\end{proof}

\pacnotimplyexact*
\begin{proof}[Sketch]
Consider the OMQA learning framework $\Fmf(L,\fixedAbox,Q)=(\examples,\set, \hypothesisSpace, \mu)$ defined in the main part (Section ``Learning from Data''). 
We have that $\hypothesisSpace$ is exponential in $n$, but finite.
Given
$\epsilon, \delta \in (0, 1)$ and $f \colon (0,1)^{2} \to \mathbb{N}$ such that $f(\epsilon, \delta) \leq \lceil \frac{1}{\epsilon}\ln\frac{|\hypothesisSpace|}{\delta} \rceil$,
and a target $\Tmc \in \hypothesisSpace$,
let
$$\sample_{m}  = \{ (e_{1},\lab{e_{1}}{\targetOntology}), \ldots, (e_{m},\lab{e_{m}}{\targetOntology})  \}$$ be a 
sample generated by ${\sf EX}^{\prob}_{\Fmf(L,\{A(a)\},Q), \Tmc}$  of size $m \geq f(\epsilon, \delta)$.
We can compute in polynomial time a hypothesis $\hypothesisOntology$ consistent with the $m$ examples  as follows.
\begin{enumerate}
\item Set
$\hypothesisOntology = 
\Tmc_{0}$. 
\item If $\exists \ssigma.M(a)$ appears in a positive example of $\sample_{m}$,
add $A \sqsubseteq \exists \ssigma.M$ to \Hmc. 
\end{enumerate}
By definition of $\hypothesisSpace$, 
at most one  example of the form $(\{A(a)\},\exists \ssigma.M(a),1)$ 
can occur in the sample.
One can verify that \Hmc is consistent with all the examples in $\sample_{m}$. 
Since $\hypothesisSpace$ is of  exponential size, 
a sample of polynomial size suffices~\cite{Vapnik:1995:NSL:211359}, and as we already argued, a  hypothesis consistent with
any sample (generated by ${\sf EX}^{\prob}_{\Fmf(L,\{A(a)\},Q), \Tmc}$) 
can be constructed in polynomial time. 
Thus, the learning framework is polynomial time PAC learnable.
On the other hand, $\Fmf(L,\fixedAbox,Q)$ is not in \PQuery~\cite[proof of Lemma~8]{DBLP:conf/aaai/KonevOW16}.
\end{proof}

 \fi

\end{document}